\documentclass{article}
\usepackage[nonatbib,preprint]{neurips_2020}
\usepackage[numbers]{natbib} 
\usepackage[utf8]{inputenc} 
\usepackage{booktabs}       
\usepackage{amsfonts}       
\usepackage{nicefrac}       
\usepackage{microtype}      
\usepackage{amssymb,relsize}
\usepackage[intlimits]{amsmath}
\usepackage{array}
\usepackage{multirow}
\usepackage{mathabx}
\usepackage{appendix}
\usepackage{amsthm}
\usepackage{graphicx,epstopdf} 
\usepackage{comment}
\usepackage{wrapfig}
\usepackage{cleveref}

\newcommand{\mat}[1]{\ensuremath \boldsymbol{\mathbf{#1}}} 
\newcommand{\order}[1]{\ensuremath{\mathcal{O}(#1)}}
\newcommand{\inR}[1]{\ensuremath \in \mathbb{R}^{#1}}
\newcommand{\nystrom}{Nystr\"{o}m}

\newcommand{\K}{\mat{K}}

\newcommand{\mPhi}{\ensuremath \mat{\Phi}}
\newcommand{\mphi}{\ensuremath \mat{\phi}}
\newcommand{\mmu}{\ensuremath \mat{\mu}}
\newcommand{\mSigma}{\ensuremath \mat{\Sigma}}
\newcommand{\meanLoss}{\ensuremath \mathcal{L}_\mu}
\newcommand{\itilde}{\ensuremath \tilde{\mat{i}}}
\newcommand{\jtilde}{\ensuremath \tilde{\mat{j}}}

\newcommand{\ltilde}{\ensuremath \tilde{\mat{\ell}}}
\newcommand{\rtilde}{\ensuremath \tilde{\mat{r}}}
\newcommand{\covLoss}{\ensuremath \mathcal{L}_\Sigma}

\newcommand{\constLoss}{\ensuremath \mathcal{L}_{\textsf{const}}}

\newcommand{\iid}{\textit{iid}}

\newtheorem{proposition}{Proposition}
\newtheorem{theorem}{Theorem}
\newtheorem{remark}{Remark}
\newcommand{\lesslines}{\looseness=-1} 

\title{Quadruply Stochastic Gaussian Processes}

\author{  
	Trefor W. Evans \\
	University of Toronto\\
	\texttt{trefor.evans@mail.utoronto.ca}
	\And
	Prasanth B. Nair \\
	University of Toronto \\
	\texttt{pbn@utias.utoronto.ca}
}

\begin{document}
\maketitle
\begin{abstract}
\lesslines
We introduce a stochastic variational inference procedure for training scalable Gaussian process~(GP) models whose per-iteration complexity is independent of both the number of training points,~$n$, and the number basis functions used in the kernel approximation,~$m$.
Our central contributions include an unbiased stochastic estimator of the evidence lower bound~(ELBO) for a Gaussian likelihood, as well as
a stochastic estimator that lower bounds the ELBO for several other likelihoods such as Laplace and logistic.
Independence of the stochastic optimization update complexity on $n$ and $m$ enables inference on huge datasets using large capacity GP models.
We demonstrate accurate inference on large classification and regression datasets using GPs and relevance vector machines with up to $m=10^7$ basis functions.
\end{abstract}
\section{Introduction}
Gaussian process (GP) modelling is a powerful Bayesian approach for classification and regression, however,
it is restricted to modestly sized datasets since training and inference require \order{n^3} time and \order{n^2} storage, where $n$ is the number of training points~\citep{rasmussen_gpml}.
This has motivated the development of sparse GP methods that use a finite number $m$ ($\ll n)$ of basis functions to reduce time and memory requirements to \order{m^2n+m^3} and \order{mn+m^2}, respectively~(e.g.~\citet{smola_sor,snelson_fitc,titsias_vfe,quin_ssgp}).
However, such techniques perform poorly if too few inducing points are used, and
computational savings are lost on complex datasets that require $m$ to be large.

\citet{wilson_kiss} approached these model capacity concerns by exploiting the structure of inducing points placed on a grid, allowing for $m > n$ while reducing computational demands over an exact GP.
This inducing point structure enables performance gains in low-dimensions,
however, time and storage complexities scale \emph{exponentially} with the dataset dimensionality, rendering the technique intractable for general learning problems unless a dimensionality reduction procedure is applied.
Comparatively, the ``variational free energy''~(VFE) Gaussian process approximation~\citep{titsias_vfe} is a more general technique that makes efficient use of inducing points by optimizing them through a variational objective in an attempt to more accurately capture the true posterior.
Many extensions have been made to this approach including a stochastic training procedure that enables dataset sub-sampling~\citep{hensman_gp_big_data}~(SVGP).
This allows the technique to be extended to large datasets, however, it is inherently restricted in the capacity of the model since it incurs a cost of \order{m^3} per iteration.

In this work, we address these concerns by introducing a novel stochastic variational inference technique whose per-iteration complexity is \order{1}~(i.e. independent of both $n$ and $m$), allowing very powerful models to be constructed on large datasets.
Note that by \emph{per-iteration} complexity, we refer to computational and storage demands at each iteration of stochastic gradient descent~(SGD) and this should not be confused with 
the expected complexity to converge to a given accuracy,
which we do not discuss.
Low per-iteration complexity is extremely valuable from a practical perspective since it does not require limiting model capacity or GP approximation accuracy to available resources~(e.g. GPU memory constraints).
We compare the per-iteration complexity of the proposed approach, QSGP, to prior work in \cref{tbl:complexity_comparision}.
Below is a summary of our main contributions
\begin{itemize}
\item For regression, an unbiased SGD scheme is developed for estimating the variational parameters whose per-iteration complexity is \emph{independent} of $n$ and $m$;
\item For classification, we develop a SGD scheme that maximizes a lower bound of the ELBO with a per-iteration complexity that is also \emph{independent} of $n$ and $m$;
\item A novel control variate is introduced to reduce the variance of the stochastic gradient approximations without affecting computational complexity; and
\item We demonstrate scalability by training powerful models on large classification and regression datasets, using up to $m=10^7$ basis functions.
\end{itemize}
The first two points comprise of novel Monte Carlo estimators for the evidence lower bound~(ELBO) with a four-level stochastic sampling approach:
we sub-sample the $n$ training examples once, and sub-sample the $m$ basis functions three times.
Because of these four levels of stochasticity, we coin the approach ``quadruply stochastic Gaussian processes''~(QSGP).
\Cref{sec:unbiased_elbo} describes QSGP for regression problems, whereas \cref{sec:elbo_lower_bound} outlines QSGP for other likelihoods, including logistic likelihoods for classification.
We conclude with numerical studies in \cref{sec:studies}, but to begin we provide a brief background on Gaussian processes and outline the matrix notations used in this paper.

\begin{table}
\begin{minipage}[c]{0.55\textwidth}
{\small
\begin{tabular}{@{}lll@{}}
\toprule
& Per-iteration Computation & Per-iteration Storage\\
\midrule
Exact GP & \order{n^3} & \order{n^2}\\
VFE~\citep{titsias_vfe} & \order{nm^2 {+} m^3} & \order{nm {+} m^2}\\
SVGP~\citep{hensman_gp_big_data} & \order{m^3} & \order{m^2}\\
QSGP & \order{1} & \order{1}\\
\bottomrule
\end{tabular}
}
\end{minipage}\hfill
\begin{minipage}[c]{0.4\textwidth}
\caption{
{Per-iteration} complexities for hyper- or variational-parameter optimization.
Storage refers to the working memory requirements per SGD iteration~(e.g. GPU memory requirements).
}
\label{tbl:complexity_comparision}
\end{minipage}
\end{table}

\vspace{-2mm}
\paragraph{Notation}
We use the notations
$\mat{h}_{i,:}$, $\mat{h}_i$ and $h_{ij}$ to denote the $i$th row, $i$th column and $ij$th element of the matrix $\mat{H}$, respectively.
Given the sets of indices $\mat{u}$ and $\mat{v}$,
$\mat{H}_{\mat{u},\mat{v}}$ denotes a matrix whose $ij$th element is given by
${h}_{u_i v_j}$.

\vspace{-2mm}
\paragraph*{Background on Gaussian Processes}
Gaussian processes (GPs) provide non-parametric prior distributions over the latent function that generated a training dataset. 
It is typically assumed that the dataset is corrupted by independent Gaussian noise with variance $\sigma^2 \geq 0$
and that the latent function is drawn from a Gaussian process with zero mean and covariance determined by the kernel $k : \mathbb{R}^{d} \times \mathbb{R}^{d} \rightarrow \mathbb{R}$.
We consider a regression problem where
$\mat{X} = \{\mat{x}_i \inR{d}\}_{i=1}^n$ and $\mat{y} \inR{n}$ denote the set of $n$ training point input locations and responses, respectively. 
Inference can be carried out at the test point, $\mat{x}_* \inR{d}$, giving the following posterior distribution of the prediction $y_* \inR{}$
\begin{align*} 
\Pr(y_*| \mat{X}, \mat{x}_*) &= \mathcal{N}\left(
y_*\ \big|\ 
\mat{k}(\mat{x}_*)^T (\K + \sigma^2 \mat{I})^{-1} \mat{y},\quad
{k}(\mat{x}_*,\mat{x}_*) - \mat{k}(\mat{x}_*)^T(\K + \sigma^2 \mat{I})^{-1}\mat{k}(\mat{x}_*)
\right), 
\end{align*}
where 
$\K \inR{n \times n}$ is the training dataset kernel covariance matrix whose $ij$th element is $k(\mat{x}_i,\mat{x}_j)$, and
$\mat{k}(\mat{x}_*) \inR{n}$ is the cross-covariance matrix between test point $\mat{x}_*$ and the training dataset such that the $i$th element is $k(\mat{x}_i,\mat{x}_*)$.

\section{Unbiased ELBO Estimator in \order{1}}
\label{sec:unbiased_elbo}
\label{sec:unbiased_estimator}
\lesslines
In this section we introduce a sparse Gaussian process model using the finite basis function approximation of a kernel,
$k(\mat{x}, \mat{z}) \approx \mat{\phi}(\mat{x})^T \mat{S}^{-1} \mat{\phi}(\mat{z})$.
This kernel directly specifies a \emph{function space} prior, however, 
we can also consider a \emph{weight space} perspective to describe the following \emph{equivalent} model:
consider a generalized linear model of the form $f(\mat{x}) = \sum_{i=1}^m w_i \phi_i(\mat{x})$,
where the weight space prior is $\Pr(\mat{w}) = \mathcal{N}(\mat{w}|\mat{0},\mat{S}^{{-}1})$, and
the likelihood is $\Pr(\mat{y}|\mat{X}, \mat{w}) = \mathcal{N}\big(\mat{y}|\mat{\Phi}\mat{w}, \sigma^2\mat{I}\big)$~\citep{rasmussen_gpml}.
In our notation, 
$\mat{S} \inR{m \times m}$ is an SPD weight prior precision matrix,
$\mat{\Phi}\inR{n\times m}$, where $\phi_{ij} = \phi_j(\mat{x}_i)$, contains the evaluations of all $m$ basis functions at all $n$ training points, and
$\mat{w} \inR{m}$ are the weights.
Clearly some choices of $\mat{S}$ and $\mPhi$ will result in better kernel approximations than others but we defer discussion about these attributes until \cref{sec:empirical_bayes}, and will continue our presentation assuming arbitrary choices.

As a result of conjugacy of the Gaussian prior, the posterior of the discussed model is also Gaussian and can be directly computed in closed form in \order{m^2n + m^3}~(e.g. see \citep{rasmussen_gpml}). 
However, we choose to instead perform variational inference which we show allows for stochastic inference procedures to scale to larger values of $m$ and $n$.
The ability to handle larger $n$ means we can work with larger datasets, while larger $m$ means a better kernel approximation and subsequently a better GP approximation.
We note that for the remainder of the paper, we focus on the weight space perspective and tailor the following overview of variational inference to this model structure.

Variational inference is a method for approximating probability densities in Bayesian statistics~\citep{jordan_vi_root, wainwright_vi_root, hoffman_svi, ranganath_blackbox_vi, blei_blackbox_vi, blei_vi_review}.
Focusing on the generalized linear model described previously, we need to compute the posterior $\Pr(\mat{w}|\mat{y},\mat{X})$ which is nominally posed as an integration problem.
Variational inference turns this task into an optimization problem.
By introducing a family of probability distributions $q(\mat{w})$ parameterized by some variational parameters, we minimize the Kullback-Leibler divergence to the exact posterior.
This equates to maximization of the evidence lower bound~(ELBO) which we can write as follows
{
\begin{align} \label{eqn:elbo}
\textsf{ELBO} &=
\mathbb{E}_{q(\mat{w})}\big[ \log \Pr(\mat{w})
+ \log \Pr(\mat{y}|\mat{X}, \mat{w}) 
-  \log q(\mat{w}) \big]
\leq \log \Pr\big(\mat{y}|\mat{X}\big),
\end{align}
}%
where the equality holds if $\Pr(\mat{w}|\mat{y},\mat{X}) = q(\mat{w})$.
Computation of the ELBO generally requires analytically intractable computations, however, \citet{challis_gaussian_vi} show that all terms of \cref{eqn:elbo} can be written in closed form as follows for 
the previously introduced generalized linear model if we choose the Gaussian variational distribution
$q(\mat{w}) = \mathcal{N}(\mat{w}|\mat{\mu},\mat{\Sigma})$, where $\mat{\mu}\inR{m}$ and $\mat{\Sigma}\inR{m\times m}$
{
\begin{align}
\label{eqn:elbo_terms}
&\textsf{ELBO}
=
-\frac{1}{2}
\underbrace{
\frac{1}{\sigma^2}\bigg(- 2\mat{y}^T\mat{\Phi}\mat{\mu} + \textsf{sum}\Big((\mat{\Phi}\mat{\mu})^2\Big)\bigg) + \mat{\mu}^T\mat{S}\mat{\mu}
}_{\meanLoss(\mmu)} 
\\
\nonumber
&
-\frac{1}{2}
\underbrace{
\frac{1}{\sigma^2}\textsf{sum}\Big((\mat{\Phi}\mat{C})^2\Big) + \textsf{tr}\big(\mat{S} \mat{\Sigma}\big) - \log\big|\mSigma\big|
}_{\covLoss(\mSigma)} 
-\frac{1}{2}
\underbrace{
\log \big|2\pi \mat{S}^{{-}1}\big|
{-} m\log(2\pi) {-} m
{+} n\log(2\pi \sigma^2) {+} \tfrac{\mat{y}^T\mat{y}}{\sigma^2}
}_{\constLoss},
\end{align}
}%
where
$(\cdot)^2$ refers to the element-wise square of the argument, and
we parameterize the variational covariance using the lower triangular Cholesky factorization $\mat{C} \inR{m \times m}$ with positive values on the diagonal such that 
$\mat{\Sigma} = \mat{C}\mat{C}^T$.
While \cref{eqn:elbo} can be maximized with respect to $\{\mat{\mu}, \mat{C}\}$ using deterministic gradient based optimization at \order{m^3 + nm^2} computations per iteration, the main contribution of this paper is a novel stochastic training procedure that reduces this complexity to \order{1} computations per iteration.
That is, we demonstrate how
to compute an unbiased estimate of the ELBO~(and its gradient) 
in a time that is \emph{independent} of both the quantity of training data~$n$, and the number of basis functions~$m$.
Leveraging this estimator allows extremely fast iterations of stochastic gradient descent~(SGD), the workhorse of many large-scale machine learning optimization problems.

We first observe that when the prior is fixed, the ELBO in \cref{eqn:elbo_terms} is additively separable in the variational mean parameters~($\mat{\mu}$) and variational covariance parameters~($\mat{C}$).
This separability enables these parameters to be estimated by solving two decoupled sub-problems which we analyze separately.

\subsection{Learning the Variational Mean}
We begin with the following result which provides a novel estimator for $\meanLoss(\mmu)$ from \cref{eqn:elbo_terms}.
\begin{theorem}
\label{thm:mu_estimator}
An unbiased estimator whose evaluation has a complexity independent of $n$ and $m$ can be written as
\begin{align*}
\meanLoss(\mmu) \approx& 
-
\tfrac{2nm}{\sigma^2\widetilde{n}\widetilde{m}}
\mat{y}_{\ltilde}^T\mPhi_{\ltilde, \itilde}\mat{\mu}_{\itilde}
+
\tfrac{nm^2}{\sigma^2 \widetilde{n}\widetilde{m}^2}
\mmu_{\jtilde}^T
\mPhi_{\ltilde, \jtilde}^T
\mPhi_{\ltilde, \itilde}
\mmu_{\itilde}
+
\tfrac{m^2}{\widetilde{m}^2}
\mmu_{\jtilde}^T \mat{S}_{\jtilde, \itilde} \mmu_{\itilde},
\end{align*}
where
$\itilde,\ \jtilde \inR{\widetilde{m}}$ both contain indices sampled uniformly from $\{1,2,\dots,m\}$,
$\ltilde \inR{\widetilde{n}}$ contains indices sampled uniformly from $\{1,2,\dots,n\}$, and
$\widetilde{m}$ and $\widetilde{n}$ are the number of Monte Carlo samples.
\end{theorem}
A proof is provided in \cref{sec:mu_estimator_proof}.
To learn $\mmu$, this estimator can be differentiated to give an~(unbiased) gradient estimate of $\meanLoss$ with respect to $\mmu$. 
These gradient estimates can then be used to perform SGD, and since the gradient estimator is unbiased and $\meanLoss$ is convex in $\mmu$, the process will converge to the unique minimizer of $\meanLoss$ provided an appropriate learning rate schedule 
is used~\citep{robbins_sgd}.

\lesslines
Using the stochastic estimate of $\meanLoss$ in \cref{thm:mu_estimator} is highly advantageous for SGD since each stochastic gradient evaluation no longer depends on $n$ or $m$.
This is a significant achievement for large problems since the complexity of evaluating this loss has decreased from $\order{nm+m^2} \rightarrow \order{\widetilde{n}\widetilde{m}+\widetilde{m}^2}$, where $\widetilde{n}$ and $\widetilde{m}$ can be chosen to be arbitrarily small~(e.g. small enough to store all matrices in GPU memory).

\subsection{Learning the Variational Covariance}
Having shown that the variational mean parameters can be found using an SGD procedure whose per-iteration complexity is independent of $m$ and $n$, the following theorem provides a similar result for the variational covariance parameters~$\mat{C}$.
\begin{theorem}
\label{thm:cov_estimator}
An unbiased estimator whose evaluation has a complexity independent of $n$ and $m$ can be written
as follows
{
\begin{align*}
\covLoss(\mat{C}) \approx&
\tfrac{m}{\widetilde{m}} \sum_{r \in \rtilde}
\tfrac{nm^2}{\sigma^2 \widetilde{n}\widetilde{m}^2}
\mat{c}_{\jtilde,r}^T
\mPhi_{\ltilde, \jtilde}^T
\mPhi_{\ltilde, \itilde}
\mat{c}_{\itilde,r}
+ 
\tfrac{m^2}{\widetilde{m}^2}
\mat{c}_{\jtilde,r}^T \mat{S}_{\jtilde, \itilde} \mat{c}_{\itilde,r}
- 
2\ \log c_{rr},
\end{align*}
}%
where
$\rtilde \inR{\widetilde{m}}$ contains indices sampled uniformly from $\{1,2,\dots,m\}$.
\end{theorem}
A proof is provided in \cref{sec:cov_estimator_proof}.
This estimator provides significant savings over the exact computation of $\covLoss$ in \cref{eqn:elbo_terms} since the complexity has decreased from 
$\order{nm^2+m^3} \rightarrow \order{\widetilde{n}\widetilde{m}^2+\widetilde{m}^3}$.
Similarly to before, this estimator can be differentiated to give an~(unbiased) gradient estimate of $\covLoss$ with respect to $\mat{C}$ that can be used for SGD.
This estimator makes use of four stochastic estimates: three over the $m$ basis functions, and one over the $n$ training examples.
Hence we call this estimator ``quadruply stochastic'' and the subsequent GP a quadruply stochastic Gaussian process~(QSGP).
As mentioned in the theorem statement, the cost of evaluating this estimator is independent of $n$ and $m$, allowing for highly flexible models to be trained on huge datasets.

In contrast to the proposed approach, modern stochastic variational inference techniques commonly make use of the same mini-batch sampling procedure over the $n$ training points that is identified here, and a second stochastic estimator is used which samples the variational distribution~\citep{hoffman_svi}.
This traditional stochastic variational inference approach does eliminate the per-iteration dependence on $n$, however, even just the act of sampling the variational distribution is at least an \order{m} operation per-iteration before even estimating the ELBO.
In our approach, we perform three levels of basis function sub-sampling that enables the cost per iteration to be independent of $m$. 
To the best of our knowledge, this is the first time in the literature an observation has been made that it is possible to perform mini-batch sampling over basis functions while carrying out stochastic variational inference.

\lesslines
One remaining concern with \cref{thm:cov_estimator} as presented is that there are \order{m^2} variational covariance parameters in $\mat{C}$ which can be expensive to store if $m$ is large.
Instead, one could consider the lower triangular $\mat{C}$ matrix to have a sparse ``chevron'' pattern depicted as 
$\mat{C} = \vcenter{\hbox{\includegraphics[width=0.04\textwidth]{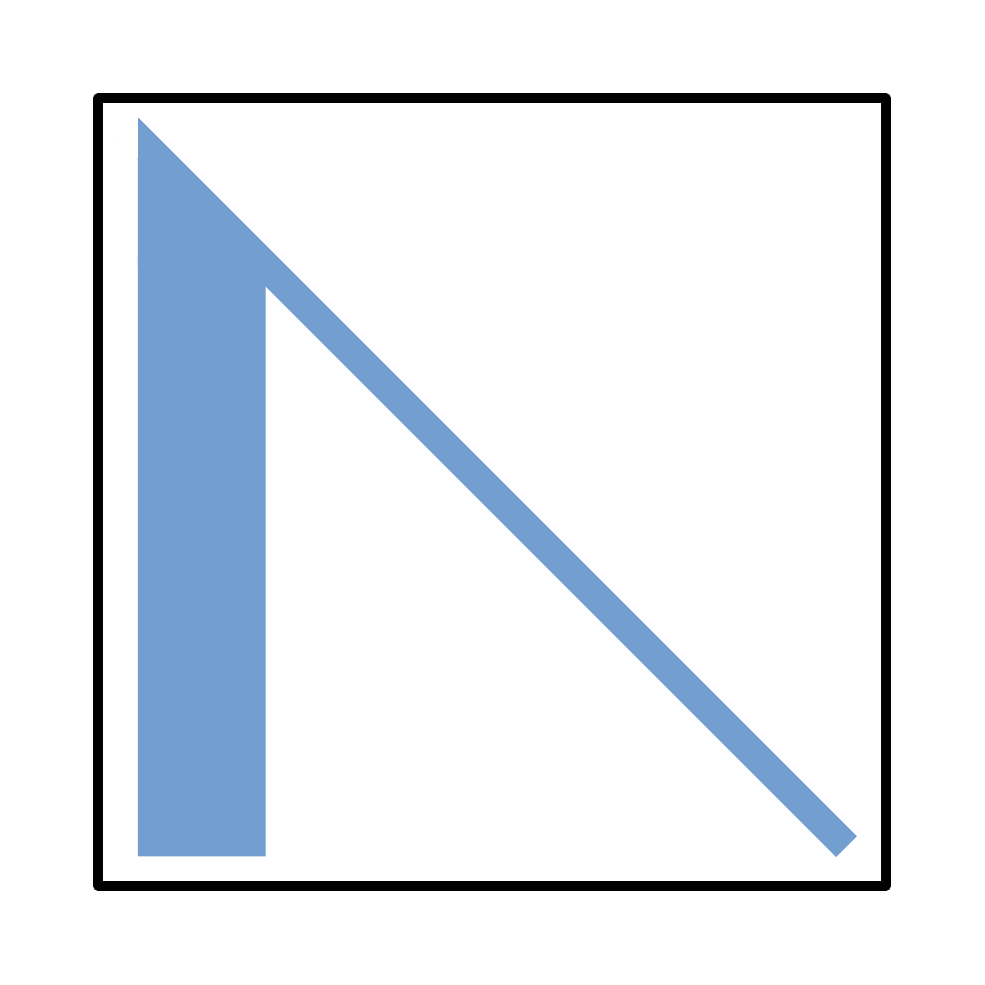}}}$, where the colour indicates non-zero elements.
This parameterization allows important posterior correlations between basis functions to be captured and since only the first few columns of $\mat{C}$ are dense, it ensures that the number of variational parameters scale as \order{m}.
While this parameterization is not invariant to permutations of basis function indexing, \citet{challis_gaussian_vi} demonstrated that it performed comparably to more complex parameterizations.
No matter whether a full or chevron parameterization is employed, 
$\covLoss$ remains convex~\citep{challis_gaussian_vi}, and
the computations in \cref{thm:cov_estimator} can be substantially reduced by exploiting the sparsity patterns of $\mat{C}$.
To further reduce the computational burden, the following result demonstrates that ${c}_{rr}$ from the diagonalized columns of the chevron Cholesky structure can be computed in closed form.
\begin{proposition}
\label{thm:crr_closed_form}
Given the parameterization $\mat{c}_r{=}\mat{e}_r c_{rr}$, where $\mat{e}_i \inR{m}$ is the $i$th unit vector, maximizing the ELBO with respect to $c_{rr}$ gives
\begin{align*}
c_{rr} = \sqrt{\frac{\sigma^2}{\mat{\phi}_r^T\mat{\phi}_r + \sigma^2 s_{rr}}}.
\end{align*}
\end{proposition}
\vspace{-3mm}
A proof is provided in \cref{sec:crr_closed_form_proof}.

\vspace{-1mm}
\subsection{Empirical Bayes}
\label{sec:empirical_bayes}
We have so far described how to learn the variational mean and covariance parameters $\mmu$ and $\mat{C}$ while keeping the prior constant.
Practitioners often choose to modify the prior by maximizing the marginal likelihood~(or evidence) with respect to a set of hyperparameters.
This is referred to as type-II inference or empirical Bayes and we discuss how this can be performed with the QSGP technique to estimate hyperparameters in the kernel $k$, which in turn affect $\mat{S}$, and $\mPhi$. 
We will perform empirical Bayes by maximizing the ELBO in \cref{eqn:elbo}, which is of course a biased surrogate for~(i.e. a lower bound of) the log-marginal likelihood, however, it is a widely used approach that performs well in practice~\citep{titsias_vfe}.
Referring to the ELBO notation in \cref{eqn:elbo_terms}, we have already shown how to efficiently estimate $\meanLoss$ and $\covLoss$, however, the term $\constLoss$ also depends on the GP prior so it must be considered as well.
The challenging term in $\constLoss$ from \cref{eqn:elbo_terms} is $\log \big|\mat{S}\big|$ which can be computed cheaply only in special cases~(e.g. if $\mat{S}$ is diagonal) but is expensive to compute in the general case.
Some kernel approximations that natively%
\footnote{Any kernel can be written to have a diagonal $\mat{S}$ using a linear transformation of features, however, performing this transformation would cost \order{m^3} in general.}
admit a diagonal $\mat{S}$ matrix include random Fourier features~\citep{rahimi_rff, quin_ssgp},
grid-structured eigenfunctions~\citep{evans_gpgrief}, and
relevance vector machines~\citep{tipping_rvm}
which allow empirical Bayes to be easily performed using the following estimator
\begin{proposition}
\label{thm:const_estimator}
Assuming a diagonal $\mat{S}$, an unbiased estimator whose evaluation complexity is independent of $n$ and $m$ can be written as follows
\begin{align*}
\constLoss \approx
- \tfrac{m}{\widetilde{m}}\sum_{i \in \itilde} \log s_{ii}
- m
+ n\ \log(2\pi \sigma^2) + 
\tfrac{n}{\sigma^2\widetilde{n}} \mat{y}_{\ltilde}^T\mat{y}_{\ltilde}.
\end{align*}
\end{proposition}
\vspace{-2mm}
\lesslines
The proof is outlined in \cref{sec:const_estimator_proof}.
We now have stochastic estimators for each term of the ELBO in \cref{eqn:elbo_terms} and have therefore demonstrated how an unbiased approximator of the ELBO can be performed in \order{1}. 
While \cref{thm:const_estimator} requires a diagonal $\mat{S}$, this estimator only needs to be included when empirical Bayes is being performed.
In the fully Bayesian case where the prior is fixed, only \cref{thm:mu_estimator,thm:cov_estimator} need to be considered, and these both assume a general, dense $\mat{S}$.
To extend the result in \cref{thm:const_estimator} to the case of a general $\mat{S}$, it may be possible to use a ``Russian roulette'' estimator following \citet{chen_stochastic_log_det} to provide a stochastic estimate of $\log|\mat{S}|$, however, this is left for future work.

\vspace{-1mm}
\subsection{Control Variates}
\label{sec:control_variate}
We discuss here how to reduce the variance of the Monte Carlo estimators introduced in \cref{thm:mu_estimator,thm:cov_estimator,thm:const_estimator}.
Variance reduction is an important consideration in practice since it affects the rate of convergence of stochastic gradient descent~(e.g. \citep{gower_sgd}).
We focus on techniques that will reduce variance while still remaining unbiased.
A classic technique to reduce variance of Monte Carlo estimators is to introduce a control variate~\citep{mcbook}.
We consider reducing the variance of the term 
$
\tfrac{nm^2}{\sigma^2\widetilde{n}\widetilde{m}^2}
\mmu_{\jtilde}^T
\mPhi_{\ltilde, \jtilde}^T
\mPhi_{\ltilde, \itilde}
\mmu_{\itilde}
$
in \cref{thm:mu_estimator}
by adding the following terms to the $\meanLoss$ estimator
\begin{align}
\label{eqn:control_variate_phiTphi}
-
\tfrac{nm^2}{\sigma^2\widebar{n}\widetilde{m}^2}
\mmu_{\jtilde}^T
\mPhi_{\mat{p}, \jtilde}^T
\mPhi_{\mat{p}, \itilde}
\mmu_{\itilde}
+
\tfrac{n}{\sigma^2\widebar{n}}
\mmu^T
\mPhi_{\mat{p}, :}^T
\mPhi_{\mat{p}, :}
\mmu,
\end{align}
where the negative control variate is the first term~(which is stochastic) and the second term is the expectation of the control variate~(which is deterministic).
The fixed set $\mat{p} \inR{\widebar{n}}$ contains indices of $\widebar{n}$ support points that are randomly sub-sampled from the training set before SGD iterations begin.
This control variate can reduce variance of the Monte Carlo estimator if there is correlation between the elements of 
$
\tfrac{n}{\widebar{n}}
\mPhi_{\mat{p}, :}^T
\mPhi_{\mat{p}, :}
$
and
$
\mPhi^T\mPhi
$.
We speculate that there would be correlation 
since the control variate uses an unbiased low-rank approximation of $\mPhi^T\mPhi$.
Additionally, it can be easily seen that the expectation of \cref{eqn:control_variate_phiTphi} is zero, therefore adding this to the \cref{thm:mu_estimator} estimator does not introduce any bias.
\Cref{eqn:control_variate_phiTphi} could also be scaled by a control variate coefficient to further reduce variance~\citep{mcbook}.

We can choose $\widebar{n} \ll n$ such that the first term in \cref{eqn:control_variate_phiTphi} can be computed cheaply, however, computation of the second term~(the control variate expectation) evidently requires $\order{m}$ computations.
The following result demonstrates how this second term can be computed in \order{1} per SGD iteration.
\begin{proposition}
\label{thm:cv_updates}
The expected value of the control variate in \cref{eqn:control_variate_phiTphi} can be computed at iteration~$t$ with a complexity independent of~$n$ and~$m$ as follows 
\begin{align*}
\tfrac{n}{\sigma^2\widebar{n}}
\mmu^T
\mPhi_{\mat{p}, :}^T
\mPhi_{\mat{p}, :}
\mmu
&=
\tfrac{n}{\sigma^2\widebar{n}}
\mat{a}^{(t)\,T}\mat{a}^{(t)},
\end{align*}
where
$\mat{a}^{(t)} = \mat{a}^{(t{-}1)} + \mPhi_{\mat{p},\itilde \cup \jtilde}\big(\mmu_{\itilde \cup \jtilde} - \mmu^{(t{-}1)}_{\itilde \cup \jtilde}\big) \inR{\widebar{n}}$,
$\mmu^{(t{-}1)} \inR{m}$ is the value of $\mmu$ at the end of iteration~$t{-}1$,
$\mat{a}^{(0)} = \mPhi_{\mat{p}, :}\mmu^{(0)}$,
$\mmu_{\itilde \cup \jtilde}$ are all the variables being updated at the current SGD iteration,
and
$\mat{a}^{(t)}$ is saved once $\mmu_{\itilde \cup \jtilde}$ has been updated at the end of iteration~$t$.
\end{proposition}
\vspace{-1mm}
This result can be easily proven by observing that $\mat{a}^{(t{-}1)} = \mPhi_{\mat{p}, :}\mmu^{(t{-}1)}$.
Clearly the cost of updating $\mat{a}^{(t)}$ at the end of iteration~$t$ requires just \order{\widebar{n}\widetilde{m}} time, and 
evidently if optimization begins at $\mmu^{(0)} =\mat{0}$, then we initialize $\mat{a}^{(0)} = \mat{0}$.
The requirement for sparse update directions can be met by simply choosing an appropriate optimizer such as regular gradient descent, or AdaGrad~\citep{adagrad}.
Computation of unbiased, sparse gradients of \cref{thm:cv_updates} is detailed in \cref{sec:control_variate_supplement}.

Finally, we note that the control variate in \cref{eqn:control_variate_phiTphi} can also be directly used in \cref{thm:cov_estimator} by simply replacing $\mmu$ with $\mat{c}_r$.
Additionally, we can derive control variates for the other terms in \cref{thm:mu_estimator,thm:cov_estimator} which we discuss in \cref{sec:control_variate_supplement}.

\section{ELBO Lower Bound Estimator in \order{1}}
\label{sec:elbo_lower_bound}
Previously we assumed that the likelihood is Gaussian, and now we generalize our results for other likelihoods so that different types of learning problems can be addressed~(e.g. classification).
In this section we aim to develop a variational bound that can be estimated in \order{1} for a wider class of likelihoods.
For many likelihoods of interest, the second term in \cref{eqn:elbo} can be written as
\begin{align}
\label{eqn:site_projection}
\mathbb{E}_{q(\mat{w})}\big[ \log \Pr(\mat{y}|\mat{X}, \mat{w}) \big]
{=}
\hspace{-0.5mm}\sum_{\ell=1}^n\mathbb{E}_{q(\mat{w})}\big[ \log g_{\ell}(\mphi_{\ell,:}\mat{w}) \big],
\end{align}
where $g_\ell: \mathbb{R} \rightarrow \mathbb{R}$ is referred to as a site projection~\citep{challis_gaussian_vi}.
Examples of site projections for different likelihoods are provided in \cref{sec:site_projection_notes}.
Additionally, since we assume $q(\mat{w})$ is Gaussian, we can re-write \cref{eqn:site_projection} as a one-dimensional expectation~\citep{barber_bishop_bnn, kuss_rasmussen_gp_classification, challis_gaussian_vi}
\begin{align}
\label{eqn:1d_expectation}
&\mathbb{E}_{q(\mat{w})}\big[ \log \Pr(\mat{y}|\mat{X}, \mat{w}) \big]
=
\sum_{\ell=1}^n\mathbb{E}_{\mathcal{N}(z|0,1)}\Big[ 
\log g_{\ell}\big(\mphi_{\ell,:}\mmu + z\:\mphi_{\ell,:}\mat{\Sigma}\mphi_{\ell,:}^T\big) \Big],
\end{align}
\lesslines
which can be easily approximated using quadrature methods when the integral is not analytically tractable.
The following result demonstrates how this likelihood expectation can be estimated in \order{1}.
\begin{theorem}
\label{thm:elbo_lower_bound}
Assuming the site projection $g_\ell$ is log-concave, the following inequality holds
\begin{align*}
&\mathbb{E}_{q(\mat{w})}\big[ \log \Pr(\mat{y}|\mat{X}, \mat{w}) \big]
\geq
\tfrac{n}{\widetilde{n}}\ \mathbb{E}\bigg[ 
\sum_{\ell \in \ltilde}\log g_{\ell}\Big(
\tfrac{m}{\widetilde{m}}\mphi_{\ell,\itilde}\mat{\mu}_{\itilde} 
+ \tfrac{m^3}{\widetilde{m}^3}z\;\mphi_{\ell,\jtilde}\mat{C}_{\jtilde,\rtilde}\mat{C}_{\itilde,\rtilde}^T\mphi_{\ell,\itilde}^T
\Big) \bigg],
\end{align*}
where the expectation on the right-hand side is taken over
$\itilde,\ \jtilde,\ \rtilde \inR{\widetilde{m}}$
, $\ltilde \inR{\widetilde{n}}$, and
$z\sim \mathcal{N}(0,1)$.
\end{theorem}
\vspace{-1mm}
The proof can be found in \cref{sec:elbo_lower_bound_proof}.
Note that we say $f(x)$ is log-concave if $\log f(x)$ is concave in $x$.
This result is clearly more general than the results presented in \cref{sec:unbiased_elbo} which were restricted to a Gaussian likelihood.
In fact, several commonly used likelihoods admit log-concave site potentials including
Gaussian likelihoods, Laplace likelihoods~(commonly used for robust regression), and logistic likelihoods~(commonly used for classification).
For Gaussian and Laplace likelihoods, the expectation over $z$ in \cref{thm:elbo_lower_bound} can be computed analytically but for a logistic likelihood it must be approximated numerically, ideally using quadrature methods~\citep{challis_gaussian_vi}.
Additionally, by extending the observations of \citet{challis_gaussian_vi} it can be shown that the ELBO approximation remains concave in the variational parameterizations discussed when the estimator in \cref{thm:elbo_lower_bound} is employed. 

As an interesting side note, maximization of the lower bound in \cref{thm:elbo_lower_bound} with $\mat{\Sigma}=\mat{0}$ is identical to a maximum likelihood learning procedure using dropout~\citep{dropout} with a dropout rate of $\tfrac{m-\widetilde{m}}{m}$.
Therefore it is evident that the dropout objective lower bounds the log-likelihood when applied to linear models~(or the final layer of a neural network) when using a likelihood with log-concave site potentials.
As a result, dropout can achieve regularization through this biasing of its original objective.

Optimizing the lower bound of the ELBO in \cref{thm:elbo_lower_bound} does introduce bias into the inference procedure, and 
it can be shown that the bias depends on the variance of the estimator over $\itilde, \jtilde,$ and~$\rtilde$.
To see this, consider an extreme example where we set $\widetilde{m}=m$ such that we perform computations with the full batch.
In this case, the estimator over $\itilde, \jtilde, \rtilde$ has zero variance, and the equality in \cref{thm:elbo_lower_bound} holds.
Evidently, the bias decreases as $\widetilde{m}$ increases and is eliminated when $\widetilde{m}=m$.

Another point of consideration is that the bias of the result in \cref{thm:elbo_lower_bound} reduces as the curvature of $\log g_\ell$ decreases~(with the bias being eliminated when $\log g_\ell$ is linear).
This result is a property of Jensen's inequality~(which was used to prove \cref{thm:elbo_lower_bound}) and 
has implications on the choice of likelihood used.
For example, the form of $\log g_\ell$ for Gaussian, Laplacian, and Logistic likelihoods are~(shifted and scaled)
quadratic, absolute value, and softplus functions, respectively~(we provide the specific forms of these in \cref{sec:site_projection_notes}).
Empirically we find that the bias is small with the Laplace and Logistic likelihoods which would be expected since both the absolute value and softplus functions are effectively piecewise linear~(the absolute value function being exactly so).
Conversely, a quadratic function generally behaves linearly nowhere and we find that the estimator in \cref{thm:elbo_lower_bound} exhibits high bias when applied to a Gaussian likelihood.
Of course, the estimators in \cref{sec:unbiased_elbo} should instead be applied when using a Gaussian likelihood since they give an unbiased approximation of the ELBO.

\vspace{-2mm}
\section{Numerical Studies}
\label{sec:studies}

\paragraph*{Classification Stress Testing}
We consider here a binary classification problem with $n=60000$ where we predict whether MNIST digits are odd or even integers.
For this problem, we consider random Fourier features to approximate a squared exponential kernel~\citep{rahimi_rff, quin_ssgp}.
Random Fourier features are attractive since they natively admit a diagonal $\mat{S}$ matrix which allows empirical Bayes to be easily performed, and 
the features are randomly generated rather than data dependent so the dataset does not need to be stored after training.
Further, storing the random features can be extremely cheap since we can just save the random seeds and regenerate them as needed~\citep{yan_doubly_stochastic_gp}.

To perform inference on this classification problem, we use the stochastic ELBO lower bound estimator in \cref{thm:elbo_lower_bound} with
a logistic likelihood,
101~quadrature points for the integral over~$z$,
$m = 10^6$ random Fourier features, and 
mini-batch sizes of $\widetilde{m} = 20000$ and $\widetilde{n}=100$.
We learn a mean-field variational distribution while simultaneously performing empirical Bayes to estimate the kernel lengthscale and variance.
The inference procedure on this huge model was performed%
\footnote{The authors' code can be found in the supplementary material.}
in just 11.1~minutes on a machine with a NVIDIA Quadro M5000 GPU, and
considering the predictive posterior median, we achieved a hold-out accuracy of $97.85\%$ and a mean negative-log-probability of 0.068 on the test set.
We note that this bests the benchmark set by \citet{hensman_svgp} when approximating the same kernel using SVGP where they achieved an accuracy and mean negative-log-probability of 
$97.8\%$ and 0.069, respectively.

To further push the capabilities of the proposed QSGP inference approach, using the same experimental setup but with $m=10^7$ further decreased the mean negative-log-probability to 0.063 on the test set.
It is promising that accurate inference can be performed while sampling only $\tfrac{\widetilde{m}}{m} = 0.2\%$ of the basis functions at each SGD iteration.\
This stress test also demonstrates our observation that once $\widetilde{m}$ is sufficiently large, increasing $m$ does not appear to effect the bias in \cref{thm:elbo_lower_bound}, nor does it seem to greatly effect the gradient variance.
Therefore we find it advisable to choose $m$ as high as possible, subject to practical constraints.
We emphasize that such a recommendation is not possible with existing models which scale poorly in $m$.
For instance, the cost of each SGD iteration is \order{m^3} for SVGP compared to \order{1} for the proposed QSGP approach.

\paragraph*{Control Variate Studies}
\begin{wrapfigure}{r}{55mm}
\centering
\vspace{-9mm}
\caption{
Effect of control variate rank~$\widebar{n}$ on the Monte Carlo variance of the objective 
$
\tfrac{1}{\sigma^2}
\mmu^T
\mPhi^T
\mPhi
\mmu
$,
and the averaged estimator gradient variance with respect to~$\mmu$.
Results are normalized.
\vspace{1mm}
}
\includegraphics[width=\linewidth]{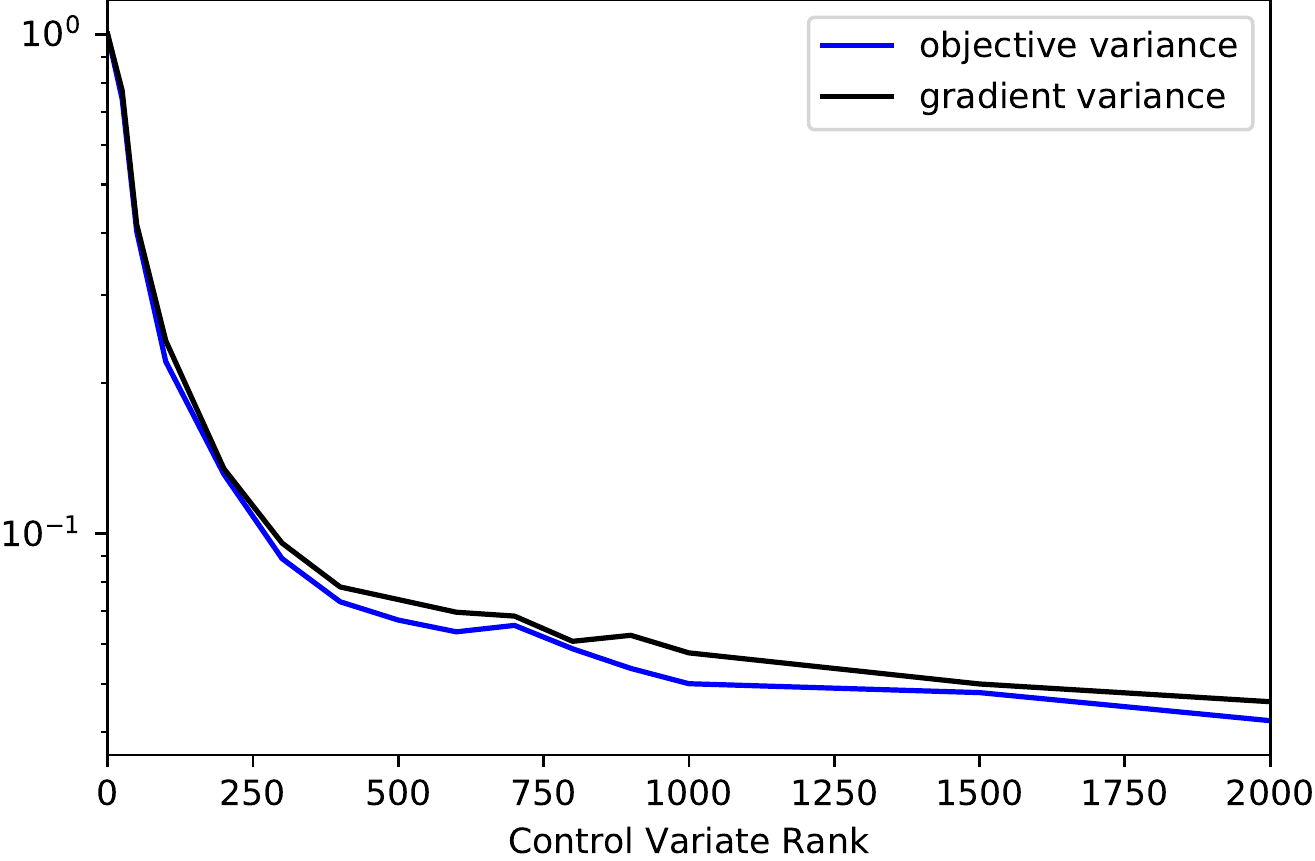}
\label{fig:control_variate}
\vspace{-10mm}
\end{wrapfigure}
\lesslines
In this section, we assess the control variate outlined in \cref{eqn:control_variate_phiTphi} and \cref{thm:cv_updates}.
To do this, we compare the variance of the Monte Carlo estimator
$
\tfrac{nm^2}{\sigma^2\widetilde{n}\widetilde{m}^2}
\mmu_{\jtilde}^T
\mPhi_{\ltilde, \jtilde}^T
\mPhi_{\ltilde, \itilde}
\mmu_{\itilde}
$
from \cref{thm:mu_estimator} both with and without the control variate.
We compare the variance of both the value of this Monte Carlo objective as well as the gradient of the estimator with respect to $\mmu$.
For the study, we consider the \emph{kin40k} dataset~($n=40000$, $d=8$) from the UCI repository, with $m=10000$ random Fourier features~\citep{rahimi_rff} to approximate the same kernel used at initialization of the regression studies of the following section.
We also set $\mmu$ to be a sample from the prior, i.e. $\mmu \sim \mathcal{N}(\mat{0}, \mat{S}^{-1})$, and we use $\widetilde{m} = \widetilde{n}=500$.
\Cref{fig:control_variate} plots both the variance of the objective, as well as the variance of its gradient with respect to $\mmu$, averaged over all elements.
These variance values are plotted with respect to the rank of the control variate used in the Monte Carlo estimator,~$\widebar{n}$.
For each $\widebar{n}$, the estimator is evaluated 1000 times to obtain the variance.
Both the objective and gradient variance are normalized with respect to the variance at $\widebar{n}=0$ where no control variate is used.
The results in \cref{fig:control_variate} show a rapid decrease in both objective and gradient variance, even with a small control variate rank.
It is promising that even though $n=40000$, when the control variate rank is just $\widebar{n}\approx 300$ the variance has already decreased by an order of magnitude.
This reduced gradient variance accelerates the convergence of SGD~(e.g. \citep{gower_sgd}).

\paragraph*{Regression Studies}
\lesslines
We consider here large regression datasets from the UCI repository using the proposed QSGP inference procedure.
We use a chevron Cholesky variational covariance parameterization,
use random Fourier features to approximate a squared exponential kernel with automatic relevance determination~\citep{rahimi_rff, quin_ssgp}, 
and we perform empirical Bayes to estimate kernel hyperparameters and the likelihood noise variance $\sigma^2$ simultaneously with the variational parameters.
The results in \cref{tbl:uci_additional} report the mean and standard deviation of the root mean squared error~(RMSE) over five test-train splits~(90\%~train, 10\%~test per split).
Also presented is the mean negative log probability~(MNLP) of the predictive posterior on the test data in the first split.
The best MNLP value of each row is in boldface whereas the best RMSE value is only in boldface if the difference is statistically significant~(if the means differ by more than three standard deviations).
QSGP-\# denotes a quadruply stochastic Gaussian process model where the first \# columns of $\mat{C}$ have a dense lower triangle~(note that QSGP-0 is a mean-field model with a diagonal $\mat{C}$).
For all QSGP models, we consider
$m=10^5$ basis functions, and we use the control variate discussed in \cref{sec:control_variate}.
We also compare to stochastic variational Gaussian processes~(SVGP) using GPFlow~\citep{hensman_svgp, gpflow} with 512 inducing points whose locations are learned along with model hyperparameters.
All models were trained on a machine with one GeForce GTX 980 Ti GPU where the maximum training time was 2.1 hours per train-test split for the QSGP-100 model trained on the \emph{ctslice} dataset.
Additional experimental details can be found in \cref{sec:additional_regression_details}.

Comparing the MNLP values between the SVGP and QSGP models, we see that the values are comparable across all datasets while the QSGP models performed noticeably better on \emph{keggu} and \emph{ctslice} caused by overconfident predictions made by SVGP.
Comparing the root-mean squared error~(RMSE) values between the SVGP and QSGP models, we also see that the values are comparable across all datasets with a few exceptions.
The largest two deviances include the \emph{ctslice} dataset where the average RMSE of the SVGP model was nearly twice that of the QSGP models, as well as the \emph{kin40k} dataset where the QSGP models also performed significantly better.

Comparing the mean-field model QSGP-0 with a diagonal $\mat{C}$ to QSGP-100 with a chevron $\mat{C}$ structure, we find that QSGP-100 generally admits lower MNLP values, suggesting a more accurate predictive posterior variance.
On some datasets such as \emph{ctslice}, this effect is quite dramatic.
This effect is expected since the chevron Cholesky parameterization allows important posterior correlations between basis functions to be captured, unlike the mean-field~(diagonal $\mat{C}$) parameterization which allows no posterior correlations to be captured and tends to give over-confident predictions.

\begin{table*}[t]
\hspace{-7mm}%
\begin{minipage}{1.08\linewidth}
	\def\arraystretch{1.1}
	{\footnotesize
		\begin{tabular}{@{}|lll|ll|ll|ll|@{}}\toprule
&&&\multicolumn{2}{c|}{QSGP-0}&\multicolumn{2}{c|}{QSGP-100}&\multicolumn{2}{c|}{SVGP}\\
Dataset&$n$&$d$&RMSE&MNLP&RMSE&MNLP&RMSE&MNLP\\
\midrule
kin40k&40000&8&$\bf 0.174 \pm 0.004$&0.135&$\bf 0.175 \pm 0.004$&0.139&$0.247 \pm 0.004$&\bf 0.055\\
protein&45730&9&$0.602 \pm 0.007$&1.34&$0.602 \pm 0.007$&1.34&$\bf 0.542 \pm 0.006$&\bf 1.06\\
kegg&48827&20&$0.135 \pm 0.004$&-0.821& $0.135 \pm 0.004$&-0.794&$0.124 \pm 0.005$&\bf -0.953\\
ctslice&53500&385&$\bf 2.484 \pm 0.107$&0.377&$\bf 2.488 \pm 0.107$&\bf -0.541&$4.746 \pm 0.087$&1.08\\
keggu&63608&27&$0.120 \pm 0.004$&0.783& $0.120 \pm 0.004$&\bf -0.268&$0.120 \pm 0.003$&36.1\\
song&515345&90&$0.492 \pm 0.002$&1.26&$0.492 \pm 0.002$&1.26&$0.488 \pm 0.002$&1.26\\
\bottomrule \end{tabular}
}
	\caption{
        Large UCI regression dataset results.
		Mean and standard deviation of RMS test error over five test-train splits, as well as MNLP of the test set from the first split.
		\vspace{-3mm} \mbox{} 
	}
	\label{tbl:uci_additional}
\end{minipage}
\end{table*}

\vspace{-2mm}
\paragraph*{Additional Studies}
Further regression studies using an alternative inducing point kernel approximation can be found in \cref{sec:regression_inducing} of the supplement.
Additionally, sparse quadruply stochastic relevance vector machines~(QSRVM) are outlined and studied in \cref{sec:rvms}.

\section{Conclusion}
\label{sec:conclusion}
The new contribution, ``Quadruply stochastic Gaussian processes''~(QSGP), demonstrates how stochastic variational inference can be applied to sparse Gaussian processes to give a per-iteration complexity that is independent of both the number of training points, and the number of basis functions that define the kernel.
The technique therefore enables Gaussian process inference to be performed on huge datasets~(large~$n$) with highly accurate kernel approximations~(since~$m$ can be made large).
A novel variance reduction strategy was also developed to accelerate SGD convergence, and the QSGP approach was demonstrated on 
large regression and classification problems with up to $m=10^7$ basis functions to show the scaling capabilities with respect to both dataset size and model capacity.

\section{Broader Impact}
Since this work is theoretical in nature there are few ethical considerations directly related to this work. 
In terms of broader impact, this work builds on a long line of work that seeks to perform accurate Gaussian process inference for large datasets.
Gaussian processes are exactly the types of models we want to apply to such problems: flexible function approximators, capable of using the information in large datasets to learn intricate structure through interpretable and expressive covariance kernels.
Unfortunately, some of these benefits are lost on modern GP techniques that only scale to large data problems at the expense of reduced model flexibility/capacity.
Conversely, the proposed approach does not require this sacrifice to be made since per-iteration SGD complexity does not depend on model capacity~(or dataset size). 
We hope that this could enable more accurate inference for large data problems which could translate to more refined uncertainty calibration in high risk applications.

\bibliographystyle{IEEEtranN} \bibliography{references} 

\newpage
\appendix
\onecolumn
\section{\Cref{thm:mu_estimator} Proof}
\label{sec:mu_estimator_proof}
\begin{proof}
The main idea of the proof is to interpret matrix operations in $\meanLoss$ as expectations, allowing us to write Monte Carlo estimators of those expectations to allow mini-batch sampling over the rows and columns of all matrices.
We begin by re-writing $\meanLoss$ in \cref{eqn:elbo_terms} purely in terms of matrix products as
\begin{align*}
\meanLoss(\mmu)
&=\frac{1}{\sigma^2}\bigg(- 2\mat{y}^T\mat{\Phi}\mat{\mu} + \textsf{sum}\Big((\mat{\Phi}\mat{\mu})^2\Big)\bigg) + \mat{\mu}^T\mat{S}\mat{\mu}
= -\tfrac{2}{\sigma^{2}}\mat{y}^T\mPhi\mat{\mu} + \tfrac{1}{\sigma^{2}}\mat{\mu}^T\mPhi^T\mPhi\mat{\mu} + \mat{\mu}^T\mat{S}\mat{\mu}.
\end{align*}
We then expand the matrix products to give a sum over $m$
\begin{align*}
\meanLoss(\mmu)
&=m\ \sum_{i=1}^m \frac{1}{m} \Big(-\tfrac{2}{\sigma^{2}}\mat{y}^T\mat{\phi}_i\mu_i + \tfrac{1}{\sigma^{2}}\mat{\mu}^T\mPhi^T\mat{\phi}_i{\mu}_i + \mat{\mu}^T\mat{s}_i\mu_i\Big),
\end{align*}
and interpreting this as an expectation, where 
$p_m(i)=m^{-1}$, $i\in\{1,2,\dots,m\}$ is a categorical probability distribution gives
\begin{align*}
\meanLoss(\mmu)
&=m\ \mathbb{E}_{i \sim p_m}\Big(-\tfrac{2}{\sigma^{2}}\mat{y}^T\mat{\phi}_i\mu_i + \tfrac{1}{\sigma^{2}}\mat{\mu}^T\mPhi^T\mat{\phi}_i{\mu}_i + \mat{\mu}^T\mat{s}_i\mu_i\Big).
\end{align*}
Finally, we write a Monte Carlo estimator for this expectation to give the following unbiased estimator
\begin{align*}
\meanLoss(\mmu)
&\approx 
- \tfrac{2m}{\widetilde{m}\sigma^{2}}\mat{y}^T\mPhi_{:,\itilde}\mmu_{\itilde} 
+ \tfrac{m}{\widetilde{m}\sigma^{2}}\mat{\mu}^T\mPhi^T\mPhi_{:,\itilde}\mmu_{\itilde} 
+ \tfrac{m}{\widetilde{m}}\mat{\mu}^T\mat{S}_{:,\itilde}\mmu_{\itilde}.
\end{align*}
The remainder of the proof follows from repeating these successive steps;
(i)~\emph{expanding} matrix operations,
(ii)~interpreting the~(expanded) matrix operations as \emph{expectations}, and
(iii)~writing \emph{Monte Carlo} estimators for these expectations.
Proceeding in this manner completes the proof,
\begingroup
\allowdisplaybreaks
\begin{align*}
\text{Expanding:} \quad
\meanLoss(\mmu)
&\approx 
- \tfrac{2m}{\widetilde{m}\sigma^{2}}\mat{y}^T\mPhi_{:,\itilde}\mmu_{\itilde} 
+ m\ \sum_{j=1}^m \frac{1}{m} \Big(
  \tfrac{m}{\widetilde{m}\sigma^{2}}\mu_j \mat{\phi}_j^T\mPhi_{:,\itilde}\mmu_{\itilde} 
+ \tfrac{m}{\widetilde{m}}\mu_j\mat{s}_{j,\itilde}\mmu_{\itilde}
\Big),
\\
\text{Expectations:} \quad
\meanLoss(\mmu)
&\approx 
- \tfrac{2m}{\widetilde{m}\sigma^{2}}\mat{y}^T\mPhi_{:,\itilde}\mmu_{\itilde} 
+ m\ \mathbb{E}_{j \sim p_m} \Big(
  \tfrac{m}{\widetilde{m}\sigma^{2}}\mu_j \mat{\phi}_j^T\mPhi_{:,\itilde}\mmu_{\itilde} 
+ \tfrac{m}{\widetilde{m}}\mu_j\mat{s}_{j,\itilde}\mmu_{\itilde}
\Big),
\\
\text{Monte Carlo:} \quad
\meanLoss(\mmu)
&\approx 
- \tfrac{2m}{\widetilde{m}\sigma^{2}}\mat{y}^T\mPhi_{:,\itilde}\mmu_{\itilde} 
+ \tfrac{m^2}{\widetilde{m}^2\sigma^{2}}\mmu_{\jtilde}^T \mPhi_{:,\jtilde}^T\mPhi_{:,\itilde}\mmu_{\itilde} 
+ \tfrac{m^2}{\widetilde{m}^2}\mmu_{\jtilde}^T\mat{S}_{\jtilde,\itilde}\mmu_{\itilde},
\\
\text{Expanding:} \quad
\meanLoss(\mmu)
&\approx 
n\ \sum_{\ell=1}^n \frac{1}{n}\Big(
- \tfrac{2m}{\widetilde{m}\sigma^{2}}y_{\ell}\mphi_{\ell,\itilde}\mmu_{\itilde} 
+ \tfrac{m^2}{\widetilde{m}^2\sigma^{2}}\mmu_{\jtilde}^T \mphi_{\ell,\jtilde}^T\mphi_{\ell,\itilde}\mmu_{\itilde}
\Big)
+ \tfrac{m^2}{\widetilde{m}^2}\mmu_{\jtilde}^T\mat{S}_{\jtilde,\itilde}\mmu_{\itilde},
\\
\text{Expectations:} \quad
\meanLoss(\mmu)
&\approx 
n\ \mathbb{E}_{\ell \sim p_n}\Big(
- \tfrac{2m}{\widetilde{m}\sigma^{2}}y_{\ell}\mphi_{\ell,\itilde}\mmu_{\itilde} 
+ \tfrac{m^2}{\widetilde{m}^2\sigma^{2}}\mmu_{\jtilde}^T \mphi_{\ell,\jtilde}^T\mphi_{\ell,\itilde}\mmu_{\itilde}
\Big)
+ \tfrac{m^2}{\widetilde{m}^2}\mmu_{\jtilde}^T\mat{S}_{\jtilde,\itilde}\mmu_{\itilde},
\\
\text{Monte Carlo:} \quad
\meanLoss(\mmu)
&\approx 
- \tfrac{2nm}{\widetilde{n}\widetilde{m}\sigma^{2}}\mat{y}_{\ltilde}^T\mPhi_{\ltilde,\itilde}\mmu_{\itilde} 
+ \tfrac{nm^2}{\widetilde{n}\widetilde{m}^2\sigma^{2}}\mmu_{\jtilde}^T \mPhi_{\ltilde,\jtilde}^T\mPhi_{\ltilde,\itilde}\mmu_{\itilde}
+ \tfrac{m^2}{\widetilde{m}^2}\mmu_{\jtilde}^T\mat{S}_{\jtilde,\itilde}\mmu_{\itilde}.
\end{align*}
\endgroup
\end{proof}

\section{\Cref{thm:cov_estimator} Proof}
\label{sec:cov_estimator_proof}
\begin{proof}
This proof proceeds similarly to that of \cref{thm:mu_estimator} where the main idea is to interpret matrix operations in $\covLoss$ as expectations, allowing us to write Monte Carlo estimators of those expectations to allow mini-batch sampling over the rows and columns of all matrices.
We begin by re-writing $\covLoss$ in \cref{eqn:elbo_terms} purely in terms of matrix products
\begin{align*}
\covLoss(\mSigma)
&=\frac{1}{\sigma^2}\textsf{sum}\Big((\mat{\Phi}\mat{C})^2\Big) + \textsf{tr}\big(\mat{S} \mat{\Sigma}\big) - \log\big|\mSigma\big|
= \sum_{r=1}^m
\tfrac{1}{\sigma^2}\mat{c}_r^T\mPhi^T\mPhi\mat{c}_r
+ \mat{c}_r^T\mat{S}\mat{c}_r
- 2\log c_{rr},
\end{align*}
interpreting the summation as an expectation, we have
\begin{align*}
\covLoss(\mSigma)
&= m\ \mathbb{E}_{r \sim p_m}\Big(
\tfrac{1}{\sigma^2}\mat{c}_r^T\mPhi^T\mPhi\mat{c}_r
+ \mat{c}_r^T\mat{S}\mat{c}_r
- 2\log c_{rr}
\Big),
\end{align*}
where
$p_m(i)=m^{-1}$, $i\in\{1,2,\dots,m\}$ is a categorical probability distribution,
and writing a Monte Carlo estimator for this expectation gives the following unbiased estimator
\begin{align*}
\covLoss(\mSigma)
&\approx
\tfrac{m}{\widetilde{m}} \sum_{r \in \rtilde}
  \tfrac{1}{\sigma^2}\mat{c}_r^T\mPhi^T\mPhi\mat{c}_r
+ \mat{c}_r^T\mat{S}\mat{c}_r
- 2\log c_{rr}.
\end{align*}
We now proceed in a manner identically to the proof of \cref{thm:mu_estimator} where we repeat three successive steps;
(i)~\emph{expanding} matrix operations in $\covLoss$,
(ii)~interpreting these~(expanded) matrix operations as \emph{expectations}, and
(iii)~writing \emph{Monte Carlo} estimators for these expectations.
Iterating,
\begingroup
\allowdisplaybreaks
\begin{align*}
\text{Expanding:} \quad
\covLoss(\mSigma)
&\approx
\tfrac{m}{\widetilde{m}} \sum_{r \in \rtilde}
m\ \sum_{i=1}^m \tfrac{1}{m}\Big(
  \tfrac{1}{\sigma^2}\mat{c}_{r}^T\mPhi^T\mphi_i{c}_{i,r}
+ \mat{c}_{r}^T\mat{s}_{i}{c}_{i,r}
\Big)
- 2\log c_{rr},
\\
\text{Expectations:} \quad
\covLoss(\mSigma)
&\approx
\tfrac{m}{\widetilde{m}} \sum_{r \in \rtilde}
m\ \mathbb{E}_{i\sim p_m}\Big(
  \tfrac{1}{\sigma^2}\mat{c}_{r}^T\mPhi^T\mphi_i{c}_{i,r}
+ \mat{c}_{r}^T\mat{s}_{i}{c}_{i,r}
\Big)
- 2\log c_{rr},
\\
\text{Monte Carlo:} \quad
\covLoss(\mSigma)
&\approx
\tfrac{m}{\widetilde{m}} \sum_{r \in \rtilde}
  \tfrac{m}{\sigma^2\widetilde{m}}\mat{c}_{r}^T\mPhi^T\mPhi_{:,\itilde}\mat{c}_{\itilde,r}
+ \tfrac{m}{\widetilde{m}}\mat{c}_{r}^T\mat{S}_{:,\itilde}\mat{c}_{\itilde,r}
- 2\log c_{rr},
\\
\text{Expanding:} \quad
\covLoss(\mSigma)
&\approx
\tfrac{m}{\widetilde{m}} \sum_{r \in \rtilde}
m\ \sum_{j=1}^m \tfrac{1}{m}\Big(
  \tfrac{m}{\sigma^2\widetilde{m}}{c}_{j,r}\mphi_j^T\mPhi_{:,\itilde}\mat{c}_{\itilde,r}
+ \tfrac{m}{\widetilde{m}}{c}_{j,r}\mat{s}_{j,\itilde}\mat{c}_{\itilde,r}
\Big)
- 2\log c_{rr},
\\
\text{Expectations:} \quad
\covLoss(\mSigma)
&\approx
\tfrac{m}{\widetilde{m}} \sum_{r \in \rtilde}
m\ \mathbb{E}_{j\sim p_m}\Big(
  \tfrac{m}{\sigma^2\widetilde{m}}{c}_{j,r}\mphi_j^T\mPhi_{:,\itilde}\mat{c}_{\itilde,r}
+ \tfrac{m}{\widetilde{m}}{c}_{j,r}\mat{s}_{j,\itilde}\mat{c}_{\itilde,r}
\Big)
- 2\log c_{rr},
\\
\text{Monte Carlo:} \quad
\covLoss(\mSigma)
&\approx
\tfrac{m}{\widetilde{m}} \sum_{r \in \rtilde}
  \tfrac{m^2}{\sigma^2\widetilde{m}^2}\mat{c}_{\jtilde,r}^T\mPhi_{:,\jtilde}^T\mPhi_{:,\itilde}\mat{c}_{\itilde,r}
+ \tfrac{m^2}{\widetilde{m}^2}\mat{c}_{\jtilde,r}^T\mat{S}_{\jtilde,\itilde}\mat{c}_{\itilde,r}
- 2\log c_{rr},
\\
\text{Expanding:} \quad
\covLoss(\mSigma)
&\approx
\tfrac{m}{\widetilde{m}} \sum_{r \in \rtilde}
n\ \sum_{\ell=1}^n \tfrac{1}{n}\Big(
  \tfrac{m^2}{\sigma^2\widetilde{m}^2}\mat{c}_{\jtilde,r}^T\mphi_{\ell,\jtilde}^T\mphi_{\ell,\itilde}\mat{c}_{\itilde,r}
\Big)
+ \tfrac{m^2}{\widetilde{m}^2}\mat{c}_{\jtilde,r}^T\mat{S}_{\jtilde,\itilde}\mat{c}_{\itilde,r}
- 2\log c_{rr},
\\
\text{Expectations:} \quad
\covLoss(\mSigma)
&\approx
\tfrac{m}{\widetilde{m}} \sum_{r \in \rtilde}
n\ \mathbb{E}_{\ell\sim p_n}\Big(
  \tfrac{m^2}{\sigma^2\widetilde{m}^2}\mat{c}_{\jtilde,r}^T\mphi_{\ell,\jtilde}^T\mphi_{\ell,\itilde}\mat{c}_{\itilde,r}
\Big)
+ \tfrac{m^2}{\widetilde{m}^2}\mat{c}_{\jtilde,r}^T\mat{S}_{\jtilde,\itilde}\mat{c}_{\itilde,r}
- 2\log c_{rr},
\\
\text{Monte Carlo:} \quad
\covLoss(\mSigma)
&\approx
\tfrac{m}{\widetilde{m}} \sum_{r \in \rtilde}
  \tfrac{nm^2}{\sigma^2\widetilde{n}\widetilde{m}^2}\mat{c}_{\jtilde,r}^T\mPhi_{\ltilde,\jtilde}^T\mPhi_{\ltilde,\itilde}\mat{c}_{\itilde,r}
+ \tfrac{m^2}{\widetilde{m}^2}\mat{c}_{\jtilde,r}^T\mat{S}_{\jtilde,\itilde}\mat{c}_{\itilde,r}
- 2\log c_{rr}.
\end{align*}
\endgroup
\end{proof}

\section{\Cref{thm:crr_closed_form} Proof}
\label{sec:crr_closed_form_proof}
\begin{proof}
The optimal $\mat{c}_r$ minimizes $\covLoss(\mat{C})$.
Therefore writing the terms of $\covLoss(\mat{C})$ from \cref{eqn:elbo_terms} that depend on $\mat{c}_r$ gives the problem statement
\begin{align*}
\underset{\mat{c}_r}{\arg \min}\ \covLoss(\mat{C}) 
= 
\underset{\mat{c}_r}{\arg \min} 
\frac{1}{\sigma^2}\mat{c}_r^T\mPhi^T\mPhi\mat{c}_r
+ \mat{c}_r^T\mat{S}\mat{c}_r
- 2\log c_{rr}.
\end{align*}
Assuming $\mat{c}_r=\mat{e}_r c_{rr}$, the solution of the preceding optimization problem can be obtained by solving the one-dimensional minimization problem
\begin{align*}
\underset{{c}_{rr}}{\min} 
\Big(
\tfrac{1}{\sigma^2}\mat{\phi}_r^T\mat{\phi}_r
+ s_{rr}
\Big) c_{rr}^2
- 2\ \log c_{rr},
\end{align*}
setting the derivative with respect to $c_{rr}$ to zero
and solving for $c_{rr}$ completes the proof,
\begin{align*}
c_{rr} = \sqrt{\frac{\sigma^2}{\mat{\phi}_r^T\mat{\phi}_r + \sigma^2 s_{rr}}}.
\end{align*}
\end{proof}

\section{\Cref{thm:const_estimator} Proof}
\label{sec:const_estimator_proof}
\begin{proof}
Re-writing $\constLoss$ from \cref{eqn:elbo_terms} assuming that $\mat{S}$ is diagonal
\begin{align*}
\constLoss 
&=
\log \big|2\pi \mat{S}^{{-}1}\big|
- m\ \log(2\pi) - m
+ n\ \log(2\pi \sigma^2) + \tfrac{1}{\sigma^2}\mat{y}^T\mat{y},
\\
&=
m\log(2\pi) - \sum_{i=1}^m \log s_{ii}
- m\ \log(2\pi) - m
+ n\ \log(2\pi \sigma^2) + 
\tfrac{1}{\sigma^2} \sum_{\ell=1}^n y_\ell^2.
\end{align*}
Cancelling terms and interpreting the sums as expectations allows us to write the following unbiased Monte Carlo estimator to complete the proof
\begin{align*}
\constLoss 
&\approx
- \tfrac{m}{\widetilde{m}}\sum_{i \in \itilde} \log s_{ii}
- m
+ n\ \log(2\pi \sigma^2) + 
\tfrac{n}{\sigma^2\widetilde{n}} \mat{y}_{\ltilde}^T\mat{y}_{\ltilde}.
\end{align*}
\end{proof}

\section{\Cref{thm:elbo_lower_bound} Proof}
\label{sec:elbo_lower_bound_proof}
\begin{proof}
Re-writing \cref{eqn:1d_expectation},
\begin{align*}
\mathbb{E}_{q(\mat{w})}\big[ \log \Pr(\mat{y}|\mat{X}, \mat{w}) \big]
&=
\sum_{\ell=1}^n\mathbb{E}_{\mathcal{N}(z|0,1)}\Big[ \log g_{\ell}\big(\mphi_{\ell,:}\mat{\mu} + z\mphi_{\ell,:}\mat{\Sigma}\mphi_{\ell,:}^T\big) \Big],
\end{align*}
and writing the expression inside the site potentials $g_\ell$ as expectations gives
\begin{align*}
\mathbb{E}_{q(\mat{w})}\big[ \log \Pr(\mat{y}|\mat{X}, \mat{w}) \big]
&=
\sum_{\ell=1}^n\mathbb{E}_{\mathcal{N}(z|0,1)}\Big[ 
\log g_{\ell}\Big(
\mathbb{E}_{i,j,r\sim p_m}\big[m\phi_{\ell,i}\mu_{i} + m^3z\phi_{\ell,j}{c}_{j,r}{c}_{i,r}\phi_{\ell,i}\big]
\Big) 
\Big],
\end{align*}
where
$p_a(i)=a^{-1}$, $i\in\{1,2,\dots,a\}$ is a categorical probability distribution.
We can also write the sum over $n$ as an expectation
\begin{align*}
\mathbb{E}_{q(\mat{w})}\big[ \log \Pr(\mat{y}|\mat{X}, \mat{w}) \big]
&=
n\ \mathbb{E}_{\mathcal{N}(z|0,1),\ \ell\sim p_n}\Big[ 
\log g_{\ell}\Big(
\mathbb{E}_{i,j,r\sim p_m}\big[m\phi_{\ell,i}\mu_{i} + m^3z\phi_{\ell,j}{c}_{j,r}{c}_{i,r}\phi_{\ell,i}\big]
\Big) 
\Big].
\end{align*}
Writing the expectations over mini-batches of size 
$\widetilde{m}$, $\widetilde{n}$ for the distributions over $p_m$ and $p_n$, respectively, gives
\begin{align*}
\mathbb{E}_{q(\mat{w})}\big[ \log \Pr(\mat{y}|\mat{X}, \mat{w}) \big]
&=
\tfrac{n}{\widetilde{n}}\ \mathbb{E}_{\mathcal{N}(z|0,1),\ \ltilde\sim p_n^{\widetilde{n}}}\bigg[ 
\sum_{\ell \in \ltilde}
\log g_{\ell}\Big(
\mathbb{E}_{\itilde,\jtilde,\rtilde \sim p_m^{\widetilde{m}}}\big[
\tfrac{m}{\widetilde{m}}\mphi_{\ell,\itilde}\mat{\mu}_{\itilde} + \tfrac{m^3}{\widetilde{m}^3}z\mphi_{\ell,\jtilde}\mat{C}_{\jtilde,\rtilde}\mat{C}_{\itilde,\rtilde}^T\mphi_{\ell,\itilde}^T
\big]
\Big) 
\bigg],
\end{align*}
where 
$p_a^{b}$ is an $b$-dimensional distribution with each dimension \iid\ according to $p_a$.
Finally, assuming that $g_\ell$ is log-concave, we apply Jensen's inequality to complete the proof
\begin{align*}
\mathbb{E}_{q(\mat{w})}\big[ \log \Pr(\mat{y}|\mat{X}, \mat{w}) \big]
&\geq
\tfrac{n}{\widetilde{n}}\ \mathbb{E}_{\mathcal{N}(z|0,1),\ \ltilde\sim p_n^{\widetilde{n}},\ \itilde,\jtilde,\rtilde \sim p_m^{\widetilde{m}}}\bigg[ 
\sum_{\ell \in \ltilde}\log g_{\ell}\Big(
\tfrac{m}{\widetilde{m}}\mphi_{\ell,\itilde}\mat{\mu}_{\itilde} + \tfrac{m^3}{\widetilde{m}^3}z\mphi_{\ell,\jtilde}\mat{C}_{\jtilde,\rtilde}\mat{C}_{\itilde,\rtilde}^T\mphi_{\ell,\itilde}^T
\Big) \bigg].
\end{align*}
\end{proof}

\section{Control Variates: Additional Details}
\label{sec:control_variate_supplement}
\subsection{Sparse Gradient Scaling}
The gradient of the expectation of the control variate outlined in \cref{thm:cv_updates} can be directly differentiated to give the exact (dense)~gradient, however, we would like sparse unbiased gradients with respect to the parameters $\mmu_{\itilde \cup \jtilde}$ that are updated at an SGD iteration.
Simply taking the relevant terms from the dense gradient and ignoring the others would introduce bias but we can scale this sparsified gradient appropriately to give unbiased gradients.
Because the gradient is sparse with just $|\itilde \cup \jtilde| \leq 2\widebar{m}$ non-zeros, this sparsified gradient needs to be scaled by $\tfrac{m}{|\itilde \cup \jtilde|}$.
This can be implemented in automatic differentiation software to give the correct zeroth derivative and an unbiased estimate of the first derivative as follows
\begin{align*}
\tfrac{m}{|\itilde \cup \jtilde|}\tfrac{n}{\sigma^2\widebar{n}}
\mat{a}^{(t)\,T}\mat{a}^{(t)}
+ 
\textsf{stop\_gradient}\bigg(\Big(1-\tfrac{m}{|\itilde \cup \jtilde|}\Big)\tfrac{n}{\sigma^2\widebar{n}}
\mat{a}^{(t)\,T}\mat{a}^{(t)}\bigg),
\end{align*}
where the argument of $\textsf{stop\_gradient}$ does not contribute to the gradient during its computation.

\subsection{Additional Control Variates}
In this section we outline additional control variates that can be used for the various terms in \cref{thm:mu_estimator,thm:cov_estimator}.
Note that all of these control variates can be used to reduce gradient variance without introducing bias and without introducing dependence on $n$ or $m$ in the cost of an SGD iteration.

We begin by presenting a control variate for the term 
$\tfrac{m^2}{\widetilde{m}^2}\mmu_{\jtilde}^T \mat{S}_{\jtilde, \itilde} \mmu_{\itilde}$ in \cref{thm:mu_estimator}.
The goal is ultimately to find a low-rank approximation to $\mat{S}$ which depends on the form of the matrix.
As a useful example, the popular class of inducing point kernel approximations give a prior precision matrix $\mat{S}$ that is a kernel covariance~(or Gram) matrix evaluated on the set of inducing points.
For notational convenience, consider the set of inducing points to be identical to the training set points in which case $n=m$ and $\mat{S}=\mat{K}$.
In the following expression, we present the negative control variate~(first term) plus the control variate's expectation~(term two) that can be simply added to the $\meanLoss$ estimator in \cref{thm:mu_estimator}
\begin{align}
\nonumber
-\tfrac{m^2}{\widetilde{m}^2} \mmu_{\itilde}^T \mat{K}(\text{X}_{\itilde},\text{U})\; \K(\text{U},\text{U})^{-1}\; \mat{K}(\text{U},\text{X}_{\jtilde})\; \mmu_{\jtilde}
+ \mmu^T\mat{K}(\text{X},\text{U})\; \K(\text{U},\text{U})^{-1}\; \mat{K}(\text{U},\text{X})\; \mmu,
\end{align}
where
$\text{U}$ is a set of $\widebar{n}$ support points in the $d$-dimensional input space, and
the notation $\mat{K}(\text{X},\text{U}) \inR{m \times \widebar{n}}$ denotes the kernel cross-covariance matrix between the sets X and U such that the $ij$th element of this matrix is $k(\mat{x}_i,\mat{u}_j)$.
Clearly this expression has an expectation of zero and so simply adding it to the stochastic estimator in \cref{thm:mu_estimator} does not introduce any bias.
This control variate has the capacity to reduce variance of the $\meanLoss$ estimator provided there is correlation between the elements of $\mat{S}=\K$ and $\mat{K}(\text{X},\text{U})\; \K(\text{U},\text{U})^{-1}\mat{K}(\text{U},\text{X}) \inR{m\times m}$.
This matrix is a \nystrom\ approximation of $\K$, a low rank kernel approximation that is widely used in the sparse GP community for kernel approximation, and preconditioning~\citep{williams_nystrom, snelson_fitc, cutajar_preconditioning, evans_gpgrief, peng_eigengp}.
The size of the set $\text{U}$ should be much smaller than $n$ and is often selected randomly from the training set~$\mat{X}$, however, a wealth of other selection strategies have been developed~\citep{smola_greedy_nystrom, drineas_nystrom, zhang_nystrom, belabbas_nystrom, kumar_nystrom_sampling, wang_nystrom, gittens_nystrom, li_nystrom, musco_leverage_nystrom}.
Empirically, it has been found that this matrix approximation is quite accurate~\citep{evans_gpgrief, musco_leverage_nystrom}, and we find that this control variate dramatically reduces variance in practice.

Unfortunately, naive evaluation of the expectation of this control variate~(the second term in the previous equation) requires \order{n}=\order{m} computations at each iteration.
However, when sparse updates are performed at each SGD iteration, an approach similar to that outlined in \cref{thm:cv_updates} can be used to decrease the per-iteration computations to \order{1} when using this control variate.
This control variate was used in the empirical studies of \cref{tbl:uci_inducing}.

We additionally note that a nearly identical control variate to the one outlined in the previous equation can be derived for the terms 
$
\tfrac{m^2}{\widetilde{m}^2}
\mat{c}_{\jtilde,r}^T \mat{S}_{\jtilde, \itilde} \mat{c}_{\itilde,r}
$
in \cref{thm:cov_estimator} by simply replacing $\mmu$ with $\mat{c}_{\itilde,r}$.

We can also easily develop a control variate for the term
$-\tfrac{2nm}{\sigma^2\widetilde{n}\widetilde{m}}\mat{y}_{\ltilde}^T\mPhi_{\ltilde, \itilde}\mat{\mu}_{\itilde}$ from \cref{thm:mu_estimator}
as follows
\begin{align*}
\tfrac{2m}{\sigma^2\widetilde{m}}\mat{b}_{\itilde}^T\mat{\mu}_{\itilde} - \tfrac{2}{\sigma^2}\mat{b}^T\mat{\mu},
\end{align*}
where the negative control variate is the first term, 
the expectation of the control variate is the second term, and
$\mat{b} = \mPhi^T\mat{y} \inR{m}$ is a vector that is precomputed before SGD iterations begin.
Similarly to \cref{thm:cv_updates} it is possible to make the cost of SGD iterations independent of $n$ and $m$ with this control variate, however, the precomputation of $\mat{b}$ costs \order{nm} which could be prohibitive.
To reduce this cost the vector $\mat{b}$ could be approximated.
We did not explore the use of this control variate in our experiments.

\subsection{Control Variates with Empirical Bayes}
If empirical Bayes is being performed, then $\mPhi$ is likely to change during optimization and so the basis functions used in \cref{eqn:control_variate_phiTphi} and in \cref{thm:cv_updates} should be fixed to the basis functions at initialization.
This same approach can be extended to the additional control variates discussed above.

\section{Predictive Posterior Augmentation}
\label{sec:augmentation}
Although the QSGP inference procedure allows $m$ to be very large, it is still finite and this degeneracy can negatively impact the quality of the predictive posterior variance.
In this section we discuss \emph{augmentation} to address this where a radial basis function is added to the QSGP model at each test point when evaluating the predictive variance~\citep{quinonero_sparse_gpm}.

The results of this section are not presented in the main body of the paper since the proposed augmentation requires a specific choice of kernel approximation, i.e. a specific choice of $\mPhi$ and $\mat{S}$.
Namely, the augmentation results assume that an inducing point approximation is made such that
$\phi_{i,j} = k(\mat{x}_i, \mat{z}_j)$ and
$s_{i,j} = k(\mat{z}_i, \mat{z}_j)$, where $\mat{z}_i \inR{d}$ for $i=1,\dots,m$ are the inducing point locations.
For ease of notation in our presentation, we assume that $n=m$ inducing points are centred on all training points such that $\mat{x}_i = \mat{z}_i$ for all $i=1,\dots,n$.
We will see that this is a sensible choice since augmentation requires the storage of the training dataset at test time, and
predictive posterior variance computations using augmentation require \order{n} time anyway, regardless of $m$.
This approach can be easily generalized for arbitrary inducing point locations.

The following result demonstrates how augmentation can be implemented affordably.
\begin{proposition} \label{thm:predictive_posterior_var_augmented}
The proposed model with augmentation gives the following predictive variance
\begin{equation*}
\mathbb{V}[y_*] = \mat{k}(\mat{x}_*)^T \mat{\Sigma} \mat{k}(\mat{x}_*)
+ \frac{\sigma^2 k(\mat{x}_*,\mat{x}_*)^2}{\mat{k}(\mat{x}_*)^T\mat{k}(\mat{x}_*) + \sigma^2k(\mat{x}_*,\mat{x}_*)},
\end{equation*}
where it is assumed that there is no posterior correlation between the augmented basis function and the basis functions in $\K$.
\end{proposition}
\begin{proof}
Let $\mat{X}_*$ denote the set of $n$ training points with an additional test point $\mat{x}_*=\mat{x}_{n+1}$ appended.
We will use $\mat{X}_*$ as inducing points at test time so that
\begin{align*}
\mPhi= \Big[\mat{K}\quad \mat{k}(\mat{x}_*)\Big],
\quad \text{and} \quad
\mat{S} = \left[\begin{array}{cc}
\mat{K} & \mat{k}(\mat{x}_*)\\
\mat{k}(\mat{x}_*)^T & k(\mat{x}_*,\mat{x}_*)
\end{array}\right].
\end{align*}
We will also denote $\mSigma',\mat{C}'\inR{n{+}1\times n{+}1}$ as the posterior covariance, and its respective Cholesky parameterization.
We can modify $\covLoss$ from \cref{eqn:elbo_terms} to give $\covLoss'$ as for this augmented model as follows
\begin{align*}
\covLoss' = \frac{1}{\sigma^2}\textsf{sum}\Big((\mat{\Phi}\mat{C}')^2\Big) - \log\big|\mSigma'\big| + \textsf{tr}\big(\mat{S} \mat{\Sigma}'\big)
= \sum_{r=1}^{n+1} \frac{1}{\sigma^2} \mat{c}_{r}^{\prime T} \big(\mPhi^T\mPhi + \sigma^2\mat{S}\big)\mat{c}_{r}^{\prime}
-2\log c'_{rr}.
\end{align*}
The assumption of no posterior correlation between the augmented basis function and the basis functions in the columns of $\K$ means that
$c'_{n+1,i} = 0$ for $i=1,2,\dots,n$, giving
\begin{align*}
\covLoss' = \covLoss\big(\mat{C}'_{1:n,1:n}\big)-2\log c'_{{n{+}1},{n{+}1}}
+\frac{1}{\sigma^2} \Big(\mat{k}(\mat{x}_*)^T\mat{k}(\mat{x}_*) + \sigma^2{k}_{x_*,x_*}\Big){c}^{\prime 2}_{{n{+}1},{n{+}1}},
\end{align*}
and
consequently the first $n$ rows and columns of $\mat{C}'$ should be selected to minimize $\covLoss$ from \cref{eqn:elbo_terms}~(and also \cref{thm:cov_estimator}) without any influence from the test point $\mat{x}_*$.
Therefore the augmented and non-augmented predictive posterior variance differ only by the contribution of $c'_{n{+}1,n{+}1}$.
Setting the derivative of $\covLoss'$ with respect to $c'_{{n{+}1},{n{+}1}}$ to zero and solving gives
\begin{align*}
\frac{2}{\sigma^2} \Big(\mat{k}(\mat{x}_*)^T\mat{k}(\mat{x}_*) + \sigma^2{k}_{x_*,x_*}\Big){c}'_{{n{+}1},{n{+}1}}
&= 
\frac{2}{c'_{{n{+}1},{n{+}1}}},
\\
c'_{{n{+}1},{n{+}1}} 
&= 
\sqrt{\frac{\sigma^2}{\mat{k}(\mat{x}_*)^T\mat{k}(\mat{x}_*) + \sigma^2{k}_{x_*,x_*}}},
\end{align*}
and using the square of this value as the posterior variance for the augmented basis function completes the proof.
\end{proof}
We make the following additional observations about the use of this augmentation strategy.
\begin{remark}
Augmentation cannot decrease the predictive posterior variance.
\end{remark}
\begin{remark} \label{thm:extrapolate_to_prior}
Augmentation will cause the predictive posterior variance to revert to the prior variance far from the training data~(where $\mat{k}(\mat{x}_*)$ approaches $\mat{0}$).
\end{remark}
These follow from the facts that
the second term in \cref{thm:predictive_posterior_var_augmented} must be non-negative, and
that the value of the second term approaches the prior variance $k(\mat{x}_*,\mat{x}_*)$ as  $\mat{k}(\mat{x}_*)$ approaches $\mat{0}$, respectively.

Further, the use of augmentation requires negligible additional work at test time since $\mat{k}(\mat{x}_*)$ is already required for both the predictive mean, and the non-augmented predictive variance.
The following study demonstrates how augmentation can improve the quality of the predictive variance far from the training data.

\subsubsection*{Augmentation Visualization}
\begin{figure}
\centering
\includegraphics[width=0.5\textwidth]{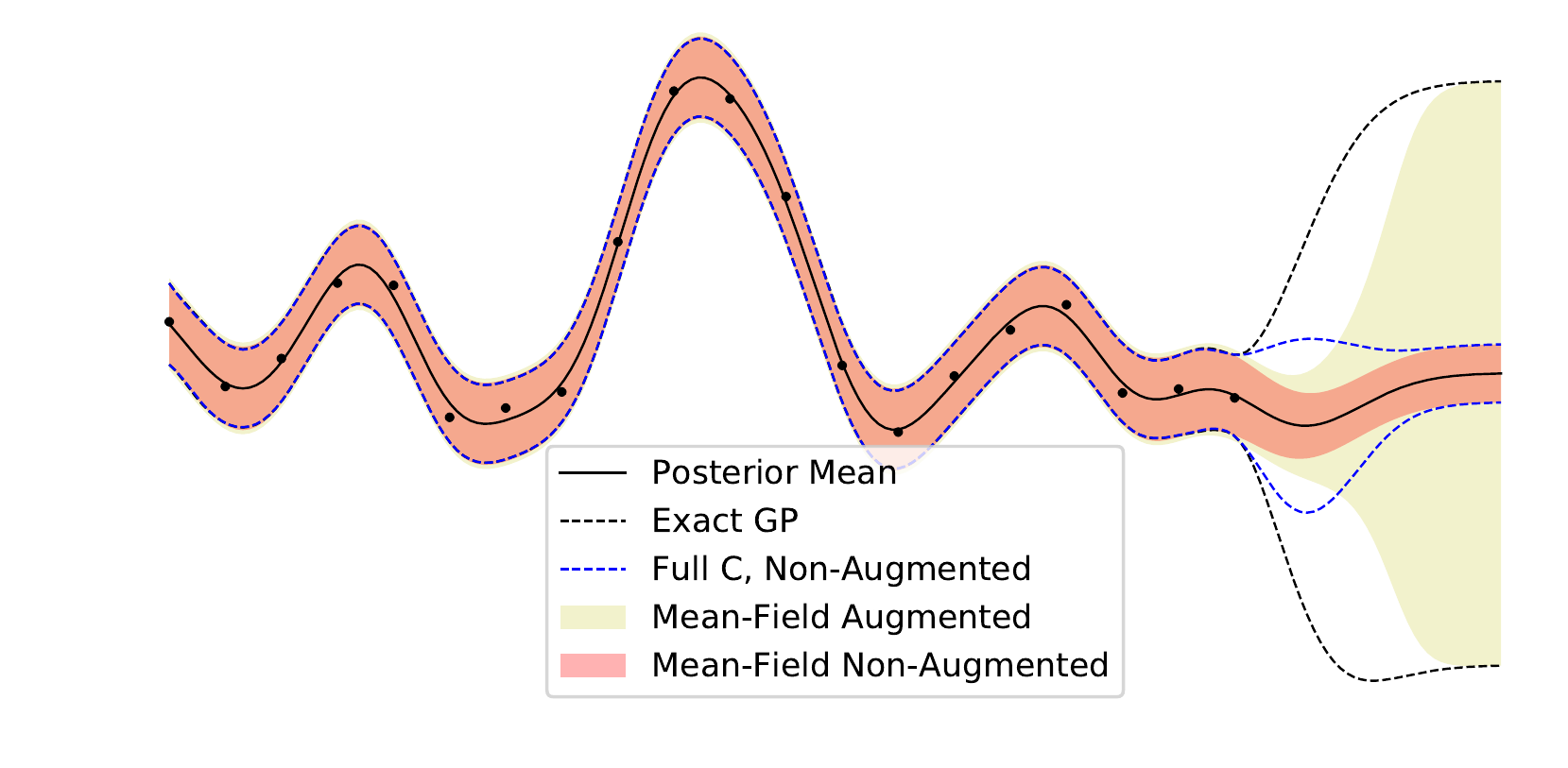}
\caption{Comparison of predictive variance between an exact GP, and a QSGP model with $\mPhi=\mat{S}=\mat{K}$ and both with and without augmentation.
}
\label{fig:augmentation_study}
\end{figure}
\Cref{fig:augmentation_study} plots the predictive posterior mean and standard deviation of an exact GP, as well as a QSGP model with $\mPhi=\mat{S}=\mat{K}$ and both with and without augmentation.
The dataset was generated using the sinc function with $n=20$ points, corrupted with independent Gaussian noise.
Note that all models have an identical posterior mean, and we use a squared exponential kernel.
Around the training data the augmented and non-augmented models give a nearly identical predictive variance that agrees well with the variance of the exact GP.
As the non-augmented model extrapolates, its predictive variance shrinks which is a highly undesirable trait of the degeneracy of the approximate GP.
However, the predictive variance of the augmented model grows as it extrapolates, and while it does not grow as fast as the exact GP, it similarly returns to the prior variance, as expected from \cref{thm:extrapolate_to_prior}.
\Cref{fig:augmentation_study} also shows the predictive variance where a dense lower triangular (or ``full'') $\mat{C}$ is considered.
Note that due to the conjugacy of the Gaussian prior, this model gives exact inference and evidently the results are better than the models constructed with a mean-field assumption where the training data ends.
However, far from the training data, the posterior variance shrinks back to that of the mean-field model.
With augmentation, the full $\mat{C}$ model would be the closest to the exact GP, however, this model is not practical for large $m$.
A mean-field or chevron Cholesky structure with augmentation could provide a compromise for large models.

\section{Site Projection Notes}
\label{sec:site_projection_notes}
We can re-write \cref{eqn:1d_expectation} as
\begin{align*}
\mathbb{E}_{q(\mat{w})}\big[ \log \Pr(\mat{y}|\mat{X}, \mat{w}) \big]
=
\sum_{\ell=1}^n\mathbb{E}_{\mathcal{N}(z|0,1)}\Big[ \log g_{\ell}\big(\mphi_{\ell,:}\mmu + z\:\mphi_{\ell,:}\mat{\Sigma}\mphi_{\ell,:}^T\big) \Big],
\end{align*}
and in this section we will outline what the form of $\log g_\ell$ is for various likelihoods.

\subsection{Logistic Likelihood}
Here we first consider Bayesian logistic regression and model the class conditional distribution using
\begin{align*}
\Pr(y=1|\mat{w}, \mat{x}) = \textsf{sig}(\mat{w}^T\mat{\phi}(\mat{x}))
\end{align*}
where $\textsf{sig}(x) = \tfrac{1}{1+\exp(-x)}$ is the sigmoid function, and 
we consider $\mat{y} \in \{-1,1\}^n$ to be the binary training labels.
Using the symmetry property of the logistic sigmoid 
$\Pr(y=-1|\mat{w}, \mat{x}) = 1 - \Pr(y=1|\mat{w}, \mat{x}) = \textsf{sig}(-\mat{w}^T\mat{\phi}(\mat{x}))$ and
assuming the data are \iid,
we can write the log-likelihood of the training set as
\begin{align*}
 \log\Pr(\mat{y}|\mat{X}, \mat{w}) = \sum_{\ell=1}^n \log\Big(\textsf{sig}(y_\ell\mat{\phi}_{\ell,:}\mat{w})\Big),
\end{align*}
where it is immediately clear from \cref{eqn:site_projection} that
$g_\ell(x) = \textsf{sig}(y_\ell x)$.
It is therefore evident that $\log g_\ell(x) = -\log\big(1+\exp(-y_\ell x)\big)$ which is a vertically flipped~(and potentially horizontally flipped) softplus function, a function is commonly used as a continuous relaxation of a rectified linear unit~(ReLU).
Additionally, it is clear by inspection that $-\log\big(1+\exp(-y_\ell x)\big)$ is a concave function in $x$ such that the result in \cref{thm:elbo_lower_bound} holds for this likelihood.

\subsection{Laplace Likelihood}
We now consider Bayesian regression using a Laplace likelihood and assume the data is \iid\ to give the training set log-likelihood
\begin{align*}
\log\Pr(\mat{y}|\mat{w}, \mat{X}) 
= \sum_{\ell=1}^n \log\;\mathcal{L}(y_{\ell}|\phi_{\ell,:}\mat{w}, b)
= \sum_{\ell=1}^n -\log(2 b) - \frac{1}{b}|y_{\ell}-\phi_{\ell,:}\mat{w}|,
\end{align*}
where
$b > 0$ is the scale parameter.
It is immediately clear from \cref{eqn:site_projection} that
$\log g_\ell(x) = -\log(2 b) - \frac{1}{b}|y_{\ell}-x|$ which a shifted absolute value function.
Additionally, it is clear by inspection that $\log g_\ell(x)$ is a concave function in $x$ such that the result in \cref{thm:elbo_lower_bound} holds for this likelihood.

The expectation over $z$ in \cref{eqn:1d_expectation} can be computed analytically for this site projection, please refer to \citep{challis_gaussian_vi} for details.

\subsection{Gaussian Likelihood}
We now consider Bayesian regression using a Gaussian likelihood and assume the data is distributed \iid\ to give the training set log-likelihood
\begin{align*}
\log\Pr(\mat{y}|\mat{w}, \mat{X}) 
= \sum_{\ell=1}^n \log\;\mathcal{N}(y_{\ell}|\phi_{\ell,:}\mat{w}, \sigma^2)
= \sum_{\ell=1}^n -\frac{1}{2}\log(2 \pi \sigma^2) - \frac{1}{2\sigma^2}(y_{\ell}-\phi_{\ell,:}\mat{w})^2,
\end{align*}
where it is immediately clear from \cref{eqn:site_projection} that
$\log g_\ell(x) = -\tfrac{1}{2}\log(2 \pi \sigma^2) - \tfrac{1}{2\sigma^2}(y_{\ell}-x)^2$ which a quadratic.
Additionally, it is clear by inspection that $\log g_\ell(x)$ is a concave function in $x$ such that the result in \cref{thm:elbo_lower_bound} holds for this likelihood.

The expectation over $z$ in \cref{eqn:1d_expectation} can be computed analytically for this site projection, please refer to \citep{challis_gaussian_vi} for details.

\section{Additional Experiments}
\label{sec:additional_experiments}
This section contains additional numerical studies to those presented in \cref{sec:studies} as well as additional experimental setup details.

\subsection{Regression Studies: Additional Details}
\label{sec:additional_regression_details}

This section contains additional experimental setup details for the regression studies of \cref{sec:studies} whose results are displayed in \cref{tbl:uci_additional}.

For all QSGP models in \cref{tbl:uci_additional}, we consider
$m=10^5$ basis functions and
mini-batch sizes of $\widetilde{m}=10000$ and $\widetilde{n}=500$.
We used the control variate outlined in \cref{thm:cv_updates} with a control variate rank of $\widebar{n}=500$ for both the
$
\tfrac{nm^2}{\sigma^2\widetilde{n}\widetilde{m}^2}
\mmu_{\jtilde}^T
\mPhi_{\ltilde, \jtilde}^T
\mPhi_{\ltilde, \itilde}
\mmu_{\itilde}
$
term in \cref{thm:mu_estimator}, and the
$
\tfrac{nm^2}{\sigma^2 \widetilde{n}\widetilde{m}^2}
\mat{c}_{\jtilde,r}^T
\mPhi_{\ltilde, \jtilde}^T
\mPhi_{\ltilde, \itilde}
\mat{c}_{\itilde,r}
$
terms in \cref{thm:cov_estimator} for each dense lower triangular column in $\mat{C}$.
Therefore, the QSGP-100 model used control variates for 101 terms in its ELBO estimator, for example.
For the inference procedure we also used an AdaGrad optimizer~\citep{adagrad} for the variational parameters with an initial learning rate of $0.1$,
an Adam optimizer~\citep{adam_optimizer} for the hyperparameters with an initial learning rate of $10^{-5}$, and
we decay both these learning rates exponentially over iterations.
We run these optimization procedures for $10^5$ iterations total, and following \citet{hensman_svgp} we find it helpful to freeze the hyperparameters for the first $10^4$ iterations until the variational parameters find a tighter ELBO.
We also compare to stochastic variational Gaussian processes~(SVGP) using GPFlow~\citep{hensman_svgp, gpflow} with 512 inducing points whose locations are learned along with model hyperparameters.
Note that the SVGP model is identical to the model presented in \cref{tbl:uci} except here the model hyperparameters are learned rather than being fixed to initial values. 
For all models, $\sigma^2$ and all kernel hyperparameters are initialized to the same values found by performing empirical Bayes with an exact GP constructed on $1000$ points randomly selected from the dataset.

\subsection{Regression Studies with Inducing Point Kernel Approximations}
\label{sec:regression_inducing}
\begin{table*}[t]
	\centering
	\def\arraystretch{1.1}
	{\footnotesize
		\begin{tabular}{@{}|lll|lll|ll|l|@{}}\toprule
&&&\multicolumn{3}{c|}{QSGP-0}&\multicolumn{2}{c|}{SVGP}&\citet{yang_a_la_carte}\\
Dataset&$n$&$d$&Time&RMSE&MNLP&RMSE&MNLP&RMSE\\
\midrule
kin40k&40000&8&1.40hrs&$\bf 0.154 \pm 0.007$&\bf -0.348&$0.171 \pm 0.003$&-0.089&$0.28\pm0.01$\\
protein&45730&9&2.13hrs&$0.598 \pm 0.007$&\bf 1.157&$0.598 \pm 0.007$&1.159&$\bf 0.53\pm0.01$\\
kegg&48827&20&1.70hrs&$0.124 \pm 0.005$&-0.554&$0.124 \pm 0.004$&\bf -0.883&$0.12\pm0.01$\\
ctslice&53500&385&5.18hrs&$2.574 \pm 0.263$&0.149&$2.938 \pm 0.098$&\bf {-0.440}&$4.00\pm0.12$\\
keggu&63608&27&1.46hrs&$0.123 \pm 0.004$&\bf 0.848&$0.118 \pm 0.004$&64.014&$0.12\pm0.00$\\
song&515345&90&3.84hrs&$0.491 \pm 0.002$&1.374&$0.491 \pm 0.002$&\bf 1.321&$0.49\pm0.00$\\
\bottomrule \end{tabular}
}
	\caption{
        Large UCI regression dataset results.
		Mean and standard deviation of RMS test error and average training time over five test-train splits, as well as MNLP of the test set from the first split.
        These results differ from those of \cref{tbl:uci_additional} in the kernel approximation used for QSGP in addition to the fact that empirical Bayes is not performed.
	}
	\label{tbl:uci}
    \label{tbl:uci_inducing}
\end{table*}

This section considers additional regression studies with an inducing point kernel approximation rather than the random Fourier feature approximation used in the regression studies of \cref{sec:studies}.

We again consider inference on large regression datasets from the UCI repository, however,
for this study, we choose to use an inducing point approximation with $n=m$ inducing points centred on each training point.
The inducing point kernel approximation results in $\mPhi = \mat{S} = \K$, the exact kernel covariance matrix between training points.
Clearly $\mat{S}=\K$ will be dense in general, therefore we do not consider empirical Bayes in the studies but keep the kernel hyperparameters fixed.
This is a natural kernel approximation choice in many ways since 
it can be shown that the approximation recovers the exact GP prior~(i.e. $\mat{y}\sim \mathcal{N}(\mat{0}, \K +\sigma^2\mat{I})$), and
the learning procedure recovers the exact GP predictive posterior mean.
While the exact GP predictive posterior variance is not recovered, we use an augmentation strategy that incurs negligible additional cost to improve the quality of the predictive variance.
This augmentation procedure is outlined in \cref{thm:predictive_posterior_var_augmented}~(of the supplementary material).

For the regression studies here, the proposed quadruply stochastic Gaussian process~(QSGP) model approximates a squared-exponential kernel with automatic relevance determination and we compare our test errors to those reported by \citet{yang_a_la_carte} on the same train-test splits where the same kernel type was approximated using Fastfood expansions.
We also compare to stochastic variational Gaussian processes~(SVGP) using GPFlow~\citep{hensman_svgp, gpflow} with 512 inducing points whose locations are learned.
For the QSGP model, a mean-field (diagonal) covariance parameterization is assumed and therefore following the notations of \cref{sec:studies}, we denote this model as QSGP-0.
Mini-batch sizes of $\widetilde{n}=\widetilde{m}=3000$ and a rank $\widebar{n}=200$ control variate was used for all datasets except for 
\emph{ctslice} which used $\widetilde{n}=\widetilde{m}=1000$.
Optimization was performed with AdaGrad~\citep{adagrad}.
For both QSGP and SVGP, $\sigma^2$ and all kernel hyperparameters were initialized by performing empirical Bayes with an exact GP constructed on $1000$ points randomly selected from the dataset.

Results are presented in \cref{tbl:uci} where we report the mean and standard deviation of the root mean squared error~(RMSE) over five test-train splits~(90\%~train, 10\%~test per split).
Also presented is the mean negative log probability~(MNLP) of the predictive posterior on the test data in the first split, and
mean training time per split on a machine with one GeForce GTX 980 Ti graphics card.
The best MNLP value of each row is in boldface whereas the best RMSE value is only in boldface if the difference is statistically significant~(if the means differ by more than three standard deviations).
The RMSE results in \cref{tbl:uci} demonstrate that QSGP performs well on these large datasets which is not surprising considering that the model has the capacity to recover the exact GP posterior mean.
Interestingly, QSGP performs nearly identically to the other sparse GP models for the \emph{kegg}, \emph{keggu}, and \emph{song} datasets.
This could indicate that the exact GP mean can be approximated well by lower capacity models for these problems. 

Considering the MNLP results, QSGP performed comparably to the SVGP model while performing noticeably better on \emph{keggu}.
While a mean-field assumption was made for these studies, we explore going beyond mean-field through the use of the chevron Cholesky structure in the regression  results of \cref{sec:studies}.


\subsection{Relevance Vector Machines}
\label{sec:rvms}
\begin{figure}
\centering
\includegraphics[width=0.5\linewidth]{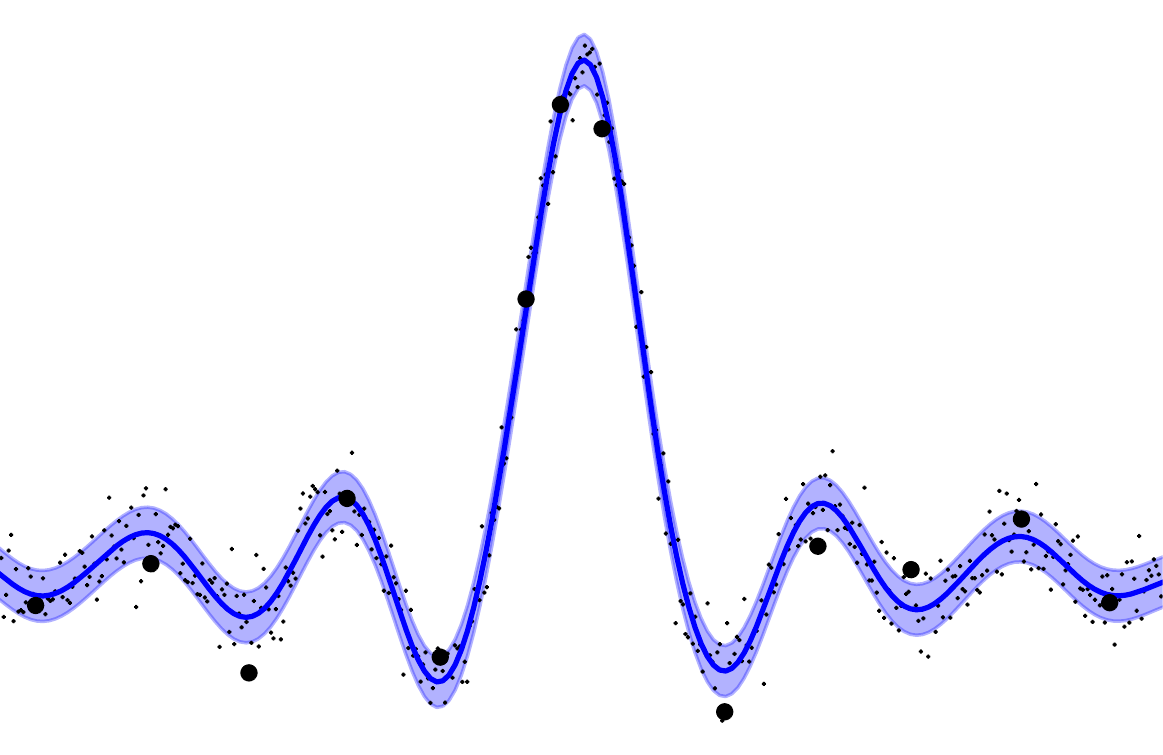}
\caption{
QSRVM fit to a noisy sinc dataset with $n=500$.
Training points are the small black dots and the relevance vectors are indicated by large black dots.
The shaded blue region is the one standard deviation from the predictive posterior mean.
}
\label{fig:qsrvm_sinc}
\end{figure}
In this section we apply the proposed quadruply stochastic inference procedure to relevance vector machines~\citep{tipping_rvm}, and we refer to this model as the quadruply stochastic relevance vector machine~(QSRVM).
QSRVMs are identical to the QSGP model, however, we parameterize $\mat{S} = \text{diag}(\mat{s})$,
where
$\mat{s} \in (0, \infty)^m$ are separate prior precision hyperparameters for each basis function.
When we maximize the evidence with respect to $\mat{s}$ by empirical Bayes, a significant fraction of them will tend to infinity and the posterior distribution over the corresponding weight parameters will be concentrated at zero, thus achieving model sparsity.
For the QSRVM model we consider $m=n$ basis functions that are kernel evaluations at all $n$ training points such that $\mPhi = \K$, the exact kernel covariance matrix between training points.
After training, the vectors $\mat{x}^{(i)}$ with $s_i$ finite are referred to as ``relevance vectors''.
In practice, we consider $s_i < 10^{4}$ to be finite.

\Cref{fig:qsrvm_sinc} demonstrates the QSRVM model trained on a noisy sinc dataset with $n=m=500$ training points~(and basis functions).
For this model, we also use a chevron Cholesky variational covariance parameterization where the first ten columns of $\mat{C}$ have a dense lower triangle.
We also use the control variate outlined in \cref{thm:cv_updates} for both the variational mean term and the terms corresponding to all ten dense columns of $\mat{C}$ with control variate rank~$\widebar{n}=250$.
For inference, we use mini-batch sizes of $\widetilde{m}=250$, and $\widetilde{n}=100$.

\Cref{fig:qsrvm_sinc} plots both the $m=n=500$ training points~(and basis functions centres), along with the locations of the discovered relevance vectors that are present the final model.
In the plot, it can be seen that just 13 relevance vectors remained in the model after training.
Evidently this model is extremely sparse.

The quadruply stochastic approach to training relevance vector machines introduces an interesting novel strategy that can allow for sparse models to be learned on huge datasets due to the fact that the per-iteration complexity doesn't depend on the number of training points or the number of basis functions in the original dictionary.
This can be a limitation for alternative RVM training techniques.
\end{document}